\documentclass{article}

\usepackage[square,sort,comma,numbers]{natbib}

\usepackage[final]{nips_2016}

\usepackage[utf8]{inputenc} 
\usepackage[T1]{fontenc}    
\usepackage{hyperref}       
\usepackage{url}            
\usepackage{booktabs}       
\usepackage{amsfonts}       
\usepackage{nicefrac}       
\usepackage{microtype}      
\usepackage{amsmath}

\usepackage{algorithm}
\usepackage{algorithmic, multicol}

\usepackage{graphicx} 
\usepackage{subfigure} 

\newcommand{\R}{{\mathbf{R}}}
\newcommand{\X}{{\mathbf{X}}}
\newcommand{\Y}{{\mathbf{Y}}}
\newcommand{\F}{{\mathbf{F}}}

\newcommand{\E}{{\mathbf{E}}}
\newcommand{\y}{{\mathbf{Y}}}

\newcommand{\p}{{\mathbf{P}}}
\newcommand{\W}{{\mathbf{W}}}

\newcommand{\Tr}{{\mathrm{Tr}}}

\newtheorem{theorem}{Theorem}[section]
\newtheorem{lemma}[theorem]{Lemma}

\newtheorem{corollary}[theorem]{Corollary}

\newenvironment{proof}[1][Proof]{\begin{trivlist}
\item[\hskip \labelsep {\bfseries #1}]}{\end{trivlist}}

\newcommand{\qed}{\nobreak \ifvmode \relax \else
      \ifdim\lastskip<1.5em \hskip-\lastskip
      \hskip1.5em plus0em minus0.5em \fi \nobreak
      \vrule height0.75em width0.5em depth0.25em\fi}
 
\renewcommand{\qed}{$\blacksquare$}

\title{Online Active Linear Regression via Thresholding}

\author{
Carlos Riquelme \\
  Stanford University \\
  \texttt{rikel@stanford.edu} \\
  \And
  Ramesh Johari \\
  Stanford University \\
  \texttt{rjohari@stanford.edu} \\
   \And
  Baosen Zhang \\
  University of Washington \\
  \texttt{zhangbao@uw.edu} \\
}

\begin{document}

\maketitle

\begin{abstract} 
We consider the problem of online active learning to collect data for regression modeling.
Specifically, we consider a decision maker with a limited experimentation budget who must efficiently learn an underlying linear population model.
Our main contribution is a novel threshold-based algorithm for selection of most informative observations; we characterize its performance and fundamental lower bounds.
We extend the algorithm and its guarantees to sparse linear regression in high-dimensional settings.
Simulations suggest the algorithm is remarkably robust: it provides significant benefits over passive random sampling in real-world datasets that exhibit high nonlinearity and high dimensionality --- significantly reducing both the mean and variance of the squared error.
\end{abstract}

\section{Introduction}
\label{intro}
This paper studies {\em online active learning} for estimation of linear models.  Active learning is motivated by the premise that in many sequential data collection scenarios, labeling or obtaining output from observations is costly.  Thus ongoing decisions must be made about whether to collect data on a particular unit of observation.
Active learning has a rich history; see, e.g., \cite{cohn1996active,cohn1994improving,castro2007minimax,koltchinskii2010rademacher,balcan2010true}.

As a motivating example, suppose that an online marketing organization plans to send display advertising promotions to a new target market.  Their goal is to estimate the revenue that can be expected for an individual with a given covariate vector.  Unfortunately, providing the promotion and collecting data on each individual is costly.  Thus the goal of the marketing organization is to acquire first the most ``informative'' observations.  They must do this in an online fashion: opportunities to display the promotion to individuals arrive sequentially over time.  In online active learning, this is achieved by selecting those observational units (target individuals in this case) that provide the most information to the model fitting procedure. 

Linear models are ubiquitous in both theory and practice---often used even in settings where the data may exhibit strong nonlinearity---in large part because of their interpretability, flexibility, and simplicity.
As a consequence, in practice, people tend to add a large number of features and interactions to the model, hoping to capture the right signal at the expense of introducing some noise.
Moreover, the input space can be updated and extended iteratively \emph{after} data collection if the decision maker feels predictions on a held-out set are not good enough.
As a consequence, often times the number of covariates becomes higher than the number of available observations.
In those cases, selecting the subsequent most informative data is even more critical.
Accordingly, our focus is on actively choosing observations for optimal \emph{prediction} of the resulting high-dimensional linear models.

Our main contributions are as follows.
We initially focus on standard linear models, and build the theory that we later extend to high dimensional settings. 
First, we develop an algorithm that sequentially selects observations if they have sufficiently large norm, in an appropriate space (dependent on the data-generating distribution).
Second, we provide a comprehensive theoretical analysis of our algorithm, including upper and lower bounds.  We focus on minimizing mean squared prediction error (MSE), and show a high probability upper bound on the MSE of our approach (cf.~Theorem \ref{th:main}).
In addition, we provide a lower bound on the best possible achievable performance in high probability and expectation (cf.~Section \ref{lower}).
In some distributional settings of interest we show that this lower bound structurally matches our upper bound, suggesting our algorithm is near-optimal.

The results above show that the improvement of active learning progressively weakens as the dimension of the data grows, and a new approach is needed.
To tackle our original goal and address this degradation, under standard sparsity assumptions, we design an adaptive extension of the thresholding algorithm that initially devotes some budget to learn the sparsity pattern of the model, in order to subsequently apply active learning to the relevant lower dimensional subspace.
We find that in this setting, the active learning algorithm provides significant benefit over passive random sampling.
Theoretical guarantees are given in Theorem \ref{th:lasso_algo}.

Finally, we empirically evaluate our algorithm's performance.
Our tests on real world data show our approach is remarkably robust: the gain of active learning remains significant even in settings that fall outside our theory. 
Our results suggest that the threshold-based rule may be a valuable tool to leverage in observation-limited environments, even when the assumptions of our theory may not exactly hold.

Active learning has mainly been studied for classification; see, e.g., \cite{balcan2006agnostic,dasgupta2007general,balcan2007margin,wang2014noise,dasgupta2008hierarchical}.  For regression, see, e.g., \cite{krause2007nonmyopic,sugiyama2009pool,cai2013maximizing} and the references within.  A closely related work to our setting is \cite{sabato2014active}: they study online or stream-based active learning for linear regression, with random design.
They propose a theoretical algorithm that partitions the space by stratification based on Monte-Carlo methods, where a recently proposed algorithm for linear regression \cite{hsu2014heavy} is used as a black box. 
It converges to the globally optimal oracle risk under possibly misspecified models (with suitable assumptions). Due to the relatively weak model assumptions, they achieve a \emph{constant} gain over passive learning.  As we adopt stronger assumptions (well-specified model), we are able to achieve larger than constant gains, with a computationally simpler algorithm.
Suppose covariate vectors are Gaussian with dimension $d$; the total number of observations is $n$; and the algorithm is allowed to label at most $k$ of them.
Then, we beat the standard $\sigma^2 d / k$ MSE to obtain $\sigma^2 d^2 / [kd + 2(\delta - 1) k \log k]$ when $n = k^\delta$, so active learning truly improves performance when $k = \Omega(\exp(d))$ or $\delta = \Omega(d)$.
While \cite{sabato2014active} does not tackle high-dimensional settings, we overcome the exponential data requirements via $l_1$-regularization.

The remainder of the paper is organized as follows.
We define our setting in Section \ref{problem}.
In Section \ref{algo}, we introduce the algorithm and provide analysis of a corresponding upper bound.
Lower bounds are given in Section \ref{lower}.  Simulations are presented in Section \ref{sim}, and Section \ref{conclusions} concludes.

\section{Problem Definition}
\label{problem}
The online active learning problem for regression is defined as follows.
We sequentially observe $n$ covariate vectors in a $d$-dimensional space $X^i \in \R^d$, which are i.i.d.
When presented with the $i$-th observation, we must choose whether we want to \emph{label} it or not, i.e., choose to observe the outcome.
If we decide to label the observation, then we obtain $Y^i \in \R$.
Otherwise, we do not see its label, and the outcome remains unknown.
We can label at most $k$ out of the $n$ observations.

We assume covariates are distributed according to some \emph{known} distribution $\mathbf{D}$, with zero mean $\E X = 0$, and covariance matrix $\Sigma = \E XX^T$.
We relax this assumption later.
In addition, we assume that $Y$ follows a linear model: $Y = X^T \beta^* + \epsilon$, where $\beta^* \in \R^d$ and $\epsilon \sim \mathcal{N}(0, \sigma^2)$ i.i.d.
We denote observations by $X, X^i \in \R^{d}$, components by $X_j \in \R$, and sets in boldface: $\X \in \R^{k \times d}, \Y \in \R^k$.

After selecting $k$ observations, $(\X, \Y)$, we output an estimate $\hat\beta_k \in \R^d$, with no intercept.\footnote{We assume covariates and outcome are centered.}
Our goal is to minimize the expected MSE of $\hat\beta_k$ in $\Sigma$ norm, i.e. $\E \| \hat\beta_k - \beta^* \|^2_{\Sigma}$, under random design; that is, when the $X_i$'s are random and the algorithm may be randomized.  This is related to the \emph{A}-optimality criterion, \cite{pukelsheim1993optimal}.
We use the experimentation budget to minimize the variance of $\hat\beta_k$ by sampling $\X$ from a different thresholded distribution.
Minimizing expected MSE is equivalent to minimizing the trace of the normalized inverse of the \emph{Fisher information matrix} $\X^T\X$,
\begin{align*}
\E[(Y - X^T \hat\beta_k)^2] &= \E[\| \hat\beta_k - \beta^* \|^2_\Sigma] + \sigma^2 \\
&= \sigma^2 \ \E\left[ \mathrm{Tr}(\Sigma (\X^T\X)^{-1}) \right] + \sigma^2
\end{align*}
where expectations are over all sources of randomness.  In this setting, the OLS estimator is the best linear unbiased estimator by the \emph{Gauss--Markov Theorem}.
Also, for any set $\X$ of $k$ i.i.d.\ observations, $\hat\beta_k := \hat\beta_k^{OLS}$ has sampling distribution $\hat\beta_k \mid \X \sim \mathcal{N}(\beta^*, \sigma^2 (\X^T\X)^{-1})$, \cite{hoerl1970ridge}.
In Section \ref{ridge}, we tackle high-dimensionality, where $k \le d$, via Lasso estimators within a two-stage algorithm.

\section{Algorithm and Main Results}
\label{algo}
In this section we motivate the algorithm, state the main result quantifying its performance for general distributions, and provide a high-level overview of the proof.
A corollary for the Gaussian distribution is presented, and we also extend the algorithm by making the threshold adaptive.
Finally, we show how to generalize the results to \emph{sparse} linear regression.
In Appendix E, we derive a CLT approximation with guarantees that is useful in complex or unknown distributional settings.

Without loss of generality, we assume that each observation is \emph{white}, that is, $\E[XX^T]$ is the identity matrix. For correlated observations $X^\prime$, we apply $X := D^{-1/2} U^T X^\prime$ to whiten them, $\Sigma = U D U^T$ (see Appendix A).
Note that $\mathrm{Tr}(\Sigma ({\X^\prime}^T\X^\prime)^{-1}) = \mathrm{Tr}((\X^T\X)^{-1})$.

We bound the whitened trace as
\begin{equation}\label{eq:trace_bound}
\frac{d}{\lambda_{\max}({\X}^T {\X})} \le \mathrm{Tr}(({\X}^T {\X})^{-1}) \le \frac{d}{\lambda_{\min}({\X}^T {\X})}.
\end{equation}
To minimize the expected MSE, we need to maximize the minimum eigenvalue of $\X^T \X$ with high probability.
The thresholding procedure in Algorithm \ref{alg:threshold} maximizes the minimum eigenvalue of $\X^T \X$ through two observations. \emph{First}, since the sum of eigenvalues of $\X^T\X$ is the trace of $\X^T\X$, which is in turn the sum of the norm of the observations, the algorithm chooses observations of large (weighted) norm. \emph{Second}, the eigenvalues of $\X^T \X$ should be balanced, that is, have similar magnitudes. This is achieved by selecting the appropriate weights for the norm.

Let $\xi \in \R^{d}_+$ be a vector of weights defining the norm $\| X \|_{\xi}^2 = {\sum_{j=1}^d \xi_j X_j^2}$. Let $\Gamma > 0$ be a threshold. Algorithm \ref{alg:threshold} simply selects the observations with $\xi$-weighted norm larger than $\Gamma$. The selected observations can be thought as i.i.d. samples from an induced distribution $\bar{\mathbf{D}}$: the original distribution conditional on $\|X \|_{\xi}\ge\Gamma $. Suppose $k$ observations are chosen and denoted by $\X \in \R^{k \times d}$. Then  $\E \X^T \X = \sum_{i=1}^k \E {X^i} {{X^i}}^T = \sum_{i=1}^k H^i = k H$, where $H$ is the covariance matrix with respect to $\bar{\mathbf{D}}$. This covariance matrix is diagonal under density symmetry assumptions, as thresholding preserves uncorrelation; its diagonal terms are
\begin{equation}
H_{jj} = \E_{\bar{\mathbf{D}}} {X}_{j}^2 = \E_{\mathbf{D}} [{X}_{j}^2 \mid \|  X \|_{\xi} \ge \Gamma] =: \phi_j.
\end{equation}
Hence,
$
\lambda_{\min}(\E \X^T\X) = k \min_j \phi_j,$ and $\lambda_{\max}(\E \X^T\X) = k \max_j \phi_j.
$
The main technical result in Theorem \ref{th:main} is to link the eigenvalues of the random matrix $\X^T\X$ to its deterministic counter part $\E \X^T \X$.
From the above calculations, the goal is to find $(\xi, \Gamma)$ such that $\min_j \phi_j \approx \max_j \phi_j$, and both are as large as possible. The first objective is achieved when there exists some $\phi$ such that
\begin{align}\label{eq:exp_threshold_cond}
\E_{\mathbf{D}} [ X_{j}^2 \mid \| X \|_{\xi} \ge \Gamma] = \phi_j = \phi, \text{ for all } j.
\end{align}
We note that if $X$ has independent components with the same marginal distribution (after whitening), then it suffices to choose $\xi_j=1$ for all $j$. It is necessary to choose unequal weights when the marginal distributions of the components are different, e.g., some are Gaussian and some are uniform, or components are dependent.
For joint Gaussian, whitening removes dependencies, so we set $\xi_j=1$.

\subsection{Thresholding Algorithm}
The algorithm is simple. For each incoming observation $X^i$ we compute its weighted norm $\| X^i \|_\xi$ (possibly after whitening if necessary).
If the norm is above the threshold $\Gamma$, then we select the observation, otherwise we ignore it.
We stop when we have collected $k$ observations. Note that \emph{random sampling} is equivalent to setting $\Gamma = 0$.

We want to catch the $k$ largest observations given our budget, therefore we require that $\Gamma$ satisfies
\begin{equation}\label{eq:prob_threshold_cond}
\p_{\mathbf{D}}\left( \| X \|_{\xi} \ge \Gamma \right) = k/n.
\end{equation}
If we apply this rule to $n$ independent observations coming from $\mathbf{D}$, on average we select $k$ of them: the $\xi-$largest.
If $(\xi, \Gamma)$ is a solution to \eqref{eq:exp_threshold_cond} and \eqref{eq:prob_threshold_cond}, then $(c \ \xi, \sqrt{c} \ \Gamma)$ is also a solution for any $c > 0$.
So we require $\sum_i \xi_i = d$.

\begin{algorithm}[ht]
\begin{algorithmic}[1]
\STATE Set $(\xi, \Gamma) \in \R^{d+1}$ satisfying \eqref{eq:exp_threshold_cond} and  \eqref{eq:prob_threshold_cond}.
\STATE Set $S = \emptyset$.
\FOR{observation $1 \le i \le n$}
\STATE Observe $X^i$.
\STATE Compute ${X^i} = {D}^{-1/2} U^T X^i$.
\IF{$\| {X^i} \|_{\xi} > \Gamma$ or $k - |S| = n - i + 1$}
\STATE Choose $X^i$: $S = S \cup X^i$.
\IF{$|S| = k$}
\STATE \textbf{break}.
\ENDIF
\ENDIF
\ENDFOR
\STATE Return OLS estimate $\hat\beta$ based on observations in $S$.
\end{algorithmic}
\caption{Thresholding Algorithm.}
\label{alg:threshold}
\end{algorithm}

Algorithm \ref{alg:threshold} can be seen as a regularizing process similar to ridge regression, where the amount of regularization depends on the distribution $\mathbf{D}$ and the budget ratio $k/n$; it improves the conditioning of the problem.

Guarantees when $\Sigma$ is unknown can be derived as follows: we allocate an initial sequence of points to estimation of the inverse of the covariance matrix, and the remainder to labeling (where we \emph{no} longer update our estimate).
In this manner observations remain independent.
Note that $O(d)$ observations are required for accurate recovery when $\mathbf{D}$ is subgaussian, and $O(d \log d)$ if subexponential, \cite{vershynin2010introduction}.
Errors by using the estimate to whiten and make decisions are bounded, small with high probability (via Cauchy--Schwarz), and the result is equivalent to using a slightly worse threshold.

\setcounter{algorithm}{0}
\begin{algorithm}[ht]
\begin{algorithmic}[1]
\STATE Set $S = \emptyset$.
\FOR{observation $1 \le i \le n$}
\STATE Observe $X^i$, estimate $\widehat{\Sigma}_i = \widehat{U}_i \widehat{D}_i \widehat{U}_i^T$.
\STATE Compute ${X^i} = \widehat{D}_i^{-1/2} \widehat{U}_i^T X_i$.
\STATE Let $(\xi_i, \Gamma_i)$ satisfy \eqref{eq:exp_threshold_cond} and \eqref{eq:prob_threshold_adaptive}.
\IF{$\| {X^i} \|_{\xi_i} > \Gamma_i$ or $k - |S|$=$n - i + 1$}
\STATE Choose $X^i$: $S = S \cup X^i$.
\IF{$|S| = k$}
\STATE \textbf{break}.
\ENDIF
\ENDIF
\ENDFOR
\STATE Return OLS estimate $\hat\beta$ based on observations in $S$.
\end{algorithmic}
\caption{\textbf{b} \ Adaptive Thresholding Algorithm.}
\label{alg:threshold_b}
\end{algorithm}

Algorithm \ref{alg:threshold} keeps the threshold fixed from the beginning, leading to a mathematically convenient analysis, as it generates i.i.d.\ observations.
However, Algorithm 1b, which is adaptive and updates its parameters after each observation, produces slightly better results, as we empirically show in Appendix K.
Before making a decision on ${X}^i$, Algorithm 1b finds $(\xi_i, \Gamma_i)$ satisfying \eqref{eq:exp_threshold_cond} and
\begin{equation}\label{eq:prob_threshold_adaptive}
\p_{\mathbf{D}}\left( \| X^i \|_{\xi_i} \ge \Gamma_i \right) = \frac{k - |S_{i-1}|}{n - i + 1},
\end{equation}
where $|S_{i-1}|$ is the number of observations already labeled.
The idea is identical: set the threshold to capture, on average, the number of observations still to be labeled, that is $k - |S_{i-1}|$, out of the number still to be observed, $n - i + 1$.

Importantly, active learning not only decreases the expected MSE, but also its variance.
Since the variance of the MSE for fixed $\X$ depends on $\sum_j 1/\lambda_j(\X^T\X)^2$ \cite{hoerl1970ridge}, it is also minimized by selecting observations that lead to large eigenvalues of $\X^T \X$.

\subsection{Main Theorem}
Theorem \ref{th:main} states that by sampling $k$ observations from $\bar{\mathbf{D}}$ where $(\xi, \Gamma)$ satisfy \eqref{eq:exp_threshold_cond}, the estimation performance is significantly improved, compared to randomly sampling $k$ observations from the original distribution. 
Section \ref{lower} shows the gain in Theorem \ref{th:main} essentially cannot be improved and Algorithm \ref{alg:threshold} is optimal.
A sketch of the proof is provided at the end of this section (see Appendix B).
\begin{theorem}\label{th:main}
Let $n > k > d$.
Assume observations $X \in \R^d$ are distributed according to subgaussian $\mathbf{D}$ with covariance matrix $\Sigma \in \R^{d \times d}$.
Also, assume marginal densities are symmetric around zero after whitening.
Let $\X$ be a $k \times d$ matrix with $k$ observations sampled from the distribution induced by the thresholding rule with parameters $(\xi, \Gamma) \in \R^{d+1}_+$ satisfying \eqref{eq:exp_threshold_cond}.
Let $\alpha > 0$, so that $t = \alpha \sqrt{k} - C \sqrt{d} > 0$, then, with probability at least $1 - 2 \exp(-c t^2)$
\begin{equation}\label{eq:th_conseq}
\mathrm{Tr}(\Sigma (\X^T\X)^{-1}) \le \frac{d}{(1 - \alpha)^2 \ \phi k},
\end{equation}
where constants $c, C$ depend on the subgaussian norm of $\bar{\mathbf{D}}$.
\end{theorem}

While Theorem \ref{th:main} is stated in fairly general terms, we can apply the result to specific settings.
We first present the Gaussian case where white components are independent. The proof is in Appendix D.
\begin{corollary}\label{co:gaussian}
If the observations in Theorem \ref{th:main} are jointly Gaussian with covariance matrix $\Sigma \in \R^{d \times d}$, $\xi_j = 1$ for all $j = 1, \dots, d$, and $\Gamma = \bar{C} \sqrt{d + 2 \log(n / k)}$, for some constant $\bar{C}  \ge 1$, then with probability at least $1 - 2 \exp(-c t^2)$ we have that
\begin{equation}
\mathrm{Tr}(\Sigma(\X^T\X)^{-1}) \le \frac{d}{(1 - \alpha)^2 \left( 1 + \frac{2 \log(n/k)}{d} \right) \ k}.
\end{equation}
\end{corollary}
The MSE of random sampling for white Gaussian data is proportional to $d/(k-d-1)$, by the inverse Wishart distribution.
Active learning provides a gain factor of order $1/(1+2\log(n/k)/d)$ with high probability (a very similar $1-\alpha$ term shows up for random sampling).
Note that our algorithm may select \emph{fewer} than $k$ observations.
Then, when the number of observations yet to be seen equals the remaining labeling budget, we should select all of them (equivalent to random sampling).
The number of observations with $\| X \|_\xi > \Gamma$ has binomial distribution, is highly concentrated around its mean $k$, with variance ${k(1-k/n)}$.
By the Chernoff Bounds, the probability that the algorithm selects fewer than $k - C^{\prime} \sqrt{k}$ decreases exponentially fast in $C^{\prime}$.
Thus, these deviations are dominated in the bound of Theorem \ref{th:main} by the leading term.
In practice, one may set the threshold in \eqref{eq:prob_threshold_cond} by choosing $k (1+\epsilon)$ observations for some small $\epsilon > 0$, or use the adaptive threshold in Algorithm 1b.


\subsection{Sparsity and Regularization}
\label{ridge}
The gain provided by active learning in our setting suffers from the curse of dimensionality, as
it diminishes very fast when $d$ increases, and Section \ref{lower} shows the gain cannot be improved in general.
For high dimensional settings (where $k \le d$) we assume $s$-\emph{sparsity} in $\beta$, that is, we assume the support of $\beta$ contains at most $s$ non-zero components, for some $s \ll d$.
In Appendix J, we also provide related results for Ridge regression.

We state the two-stage \emph{Sparse Thresholding} Algorithm (see Algorithm \ref{alg:threshold_lasso}) and show this algorithm effectively overcomes the curse of dimensionality. For simplicity, we assume the data is Gaussian, $\mathbf{D} = \mathcal{N}(0, \Sigma)$.
Based, for example, on the results of \cite{tropp2005signal} and Theorem 1 in \cite{joseph2013variable}, we could extend our results to subgaussian data via the Orthogonal Matching Pursuit algorithm for recovery.
The two-stage algorithm works as follows.
First, we focus on recovering the true support, $S = S(\beta)$, by selecting the very first $k_1$ observations (without thresholding), and computing the Lasso estimator $\hat\beta_1$.
Second, we assign the weights $\xi$: for $i \in S(\hat\beta_1)$, we set $\xi_i = 1$, otherwise we set $\xi_i = 0$.
Then, we apply the thresholding rule to select the remaining $k_2 = k - k_1$ observations.
While observations are collected in all dimensions, our final estimate $\hat\beta_2$ is the OLS estimator computed only including the observations selected in the second stage, and exclusively in those dimensions in $S(\hat\beta_1)$.

Note that, in general, the points that end up being selected by our algorithm are informational outliers, while not necessarily geometric outliers in the original space.
After applying the whitening transformation, ignoring some dimensions based on the Lasso results, and then thresholding based on a weighted norm possibly learnt from data (say, if components are not independent, and we recover the covariance matrix in a online fashion), the algorithm is able to identify good points for the underlying data distribution and $\beta$.

\begin{algorithm}[ht]
\begin{algorithmic}[1]
\STATE Set $S_1 = \emptyset, S_2 = \emptyset$. Let $k = k_1 + k_2, n = k_1 + n_2$.
\FOR{observation $1 \le i \le k_1$}
\STATE Observe $X^i$. Choose $X^i$: $S_1 = S_1 \cup X^i$.
\ENDFOR
\STATE Set $\gamma = 1/2, \lambda = \sqrt{4 \sigma^2 \log(d) / \gamma^2 k_1}$.
\STATE Compute Lasso estimate $\hat\beta_1$ on $S_1$, regularization $\lambda$.
\STATE Set weights: $\xi_i = 1$ if $i \in S(\hat\beta_1)$, $\xi_i = 0$ otherwise.
\STATE Set $\Gamma = C \sqrt{s + 2 \log(n_2 / k_2)}$.
\STATE Factorize $\Sigma_{S(\hat\beta_1)S(\hat\beta_1)} = U D U^T$.
\FOR{observation $k_1 + 1 \le i \le n$}
\STATE Observe $X^i \in \R^d$. Restrict to $X^i_S := X^i_{S(\hat\beta_1)} \in \R^s$.
\STATE Compute ${X^i}_S = {D}^{-1/2} U^T X^i_S$.
\IF{$\| {X^i_S} \|_{\xi} > \Gamma$ or $k_2 - |S_2| = n - i + 1$}
\STATE Choose $X^i_S$: $S_2 = S_2 \cup X^i_S$.
\IF{$|S_2| = k_2$}
\STATE \textbf{break}.
\ENDIF
\ENDIF
\ENDFOR
\STATE Return OLS estimate $\hat\beta_2$ based on observations in $S_2$.
\end{algorithmic}
\caption{Sparse Thresholding Algorithm.}
\label{alg:threshold_lasso}
\end{algorithm}

Theorem \ref{th:lasso_algo} summarizes the performance of Algorithm \ref{alg:threshold_lasso}; it requires the standard assumptions on $\Sigma, \lambda$ and $\min_i |\beta_i|$ for support recovery (see Theorem 3 in \cite{wainwright2009sharp}).
\begin{theorem}
\label{th:lasso_algo}
Let $\mathbf{D} = \mathcal{N}(0, \Sigma)$.
Assume $\Sigma, \lambda$ and $\min_i |\beta_i|$ satisfy the standard conditions given in Theorem 3 of \cite{wainwright2009sharp}.
Assume we run the Sparse Thresholding algorithm with $k_1 = C^{\prime} s \log d$ observations to recover the support of $\beta$, for an appropriate $C^{\prime} \ge 0$.
Let $\X_2$ be $k_2 = k - k_1$ observations sampled via thresholding on $S(\hat\beta_1)$.
It follows that for $\alpha > 0$ such that $t = \alpha \sqrt{k_2} - C \sqrt{s} > 0$, there exist some universal constants $c_1, c_2$, and $c, C$ that depend on the subgaussian norm of $\bar{\mathbf{D}} \mid S(\hat{\beta}_1)$, such that with probability at least
$$1 - 2e^{- \min \left( c_2 \min(s, \log(d-s)) - \log(c_1), ct^2 - \log(2) \right)}$$
it holds that 
\begin{equation*}\label{th:lasso_algo_perf}
\mathrm{Tr}(\Sigma_{SS}(\X_2^T\X_2)^{-1}) \le \frac{s}{(1 - \alpha)^2 \left( 1 + \frac{2 \log\left({n_2}/{k_2}\right)}{s} \right) \ k_2}.
\end{equation*}
\end{theorem}
Performance for random sampling with the Lasso estimator is $O(s \log d/k)$.
A regime of interest is $s \ll d, k = C_1 s \log d$, and $n = C_2 \ d$, for large enough $C_1$, and $C_2 > 0$.
In that case, Algorithm \ref{alg:threshold_lasso} leads to a bound of order smaller than $1 / \log(d)$, as opposed to a weaker constant guarantee for random sampling.
The gain is at least a $\log d$ factor with high probability.
The proof is in Appendix H.
In practice, the performance of the algorithm is improved by using all the $k$ observations to fit the final estimate $\hat\beta_2$, as shown in simulations. However, in that case, observations are no longer i.i.d.\
Also, using thresholding to select the initial $k_1$ observations decreases the probability of making a mistake in support recovery.
In Section \ref{sim} we provide simulations comparing different methods.

\subsection{Proof of Theorem \ref{th:main}}
The complete proof of Theorem \ref{th:main} is in Appendix B.
We only provide a sketch here.
The proof is a direct application of spectral results in \cite{vershynin2010introduction}, which are derived via a covering argument using a discrete net $\mathcal{N}$ on the unit Euclidean sphere $S^{d-1}$, together with a Bernstein-type concentration inequality that controls deviations of $\| \X w \|_2$ for each element $w \in \mathcal{N}$ in the net.
Finally, a union bound is taken over the net.
Importantly, the proof shows that if our algorithm uses $(\xi, \Gamma)$ which are \emph{approximate} solutions to \eqref{eq:exp_threshold_cond},  then \eqref{eq:th_conseq} still holds with $\min_j \E_{\bar{\mathbf{D}}} X_j^2$ in the denominator of the RHS, instead of $\phi$.
This fact can be quite useful in practice, when $\F$ is unknown.
We can devote some initial budget $X_1, \dots, X_T$ to recover $\F$, and then find $(\xi, \Gamma)$ approximately solving \eqref{eq:exp_threshold_cond} and \eqref{eq:prob_threshold_cond} under $\hat\F$.
Note that no labeling is required.

Also, the result can be extended to subexponential distributions.
In this case, the probabilistic bound will be weaker (including a $d$ term in front of the exponential).
More generally, our probabilistic bounds are strongest when $k \ge C d \log d$ for some constant $C \ge 0$, a common situation in active learning \cite{sabato2014active}, where super-linear requirements in $d$ seem unavoidable in noisy settings.
A simple bound for the parameter $\phi$ can be calculated as follows.
Assume there exists $(\xi, \Gamma)$ such that $\phi_j = \phi$ and consider the weighted squared norm $Z_\xi = \sum_{j=1}^d \xi_j X_{j}^2$.
Then
$ \E_{\bar{D}} \left[ Z_\xi \right] = \sum_{j=1}^d \xi_j \E_{\bar{D}} \left[ X_{j}^2 \right] = \sum_{j=1}^d \xi_j \phi_j = d \phi,$ and
$
\phi = \E_{{D}} \left[ Z_\xi \mid Z_\xi \ge \Gamma^2 \right] / d \ge {\Gamma^2}/{d} = {F^{-1}_{Z_\xi}(1 - k/n)}/{d},
$
which implies that $1 / \lambda_{\min}(\E \X^T\X) = 1 / k \phi \le d / k \Gamma^2$.
For specific distributions, $\Gamma^2/d$ can be easily computed.
The last inequality is close to equality in cases where the conditional density decays extremely fast for values of $\sum_{j=1}^d \xi_j X_j^2$ above $\Gamma^2$.
Heavy-tailed distributions allocate mass to significantly higher values, and $\phi$ could be much larger than $\Gamma^2 / d$.

 \begin{figure*}[t!]
  \centering
  \subfigure[Zooming out.]{\includegraphics[width=0.49 \columnwidth]{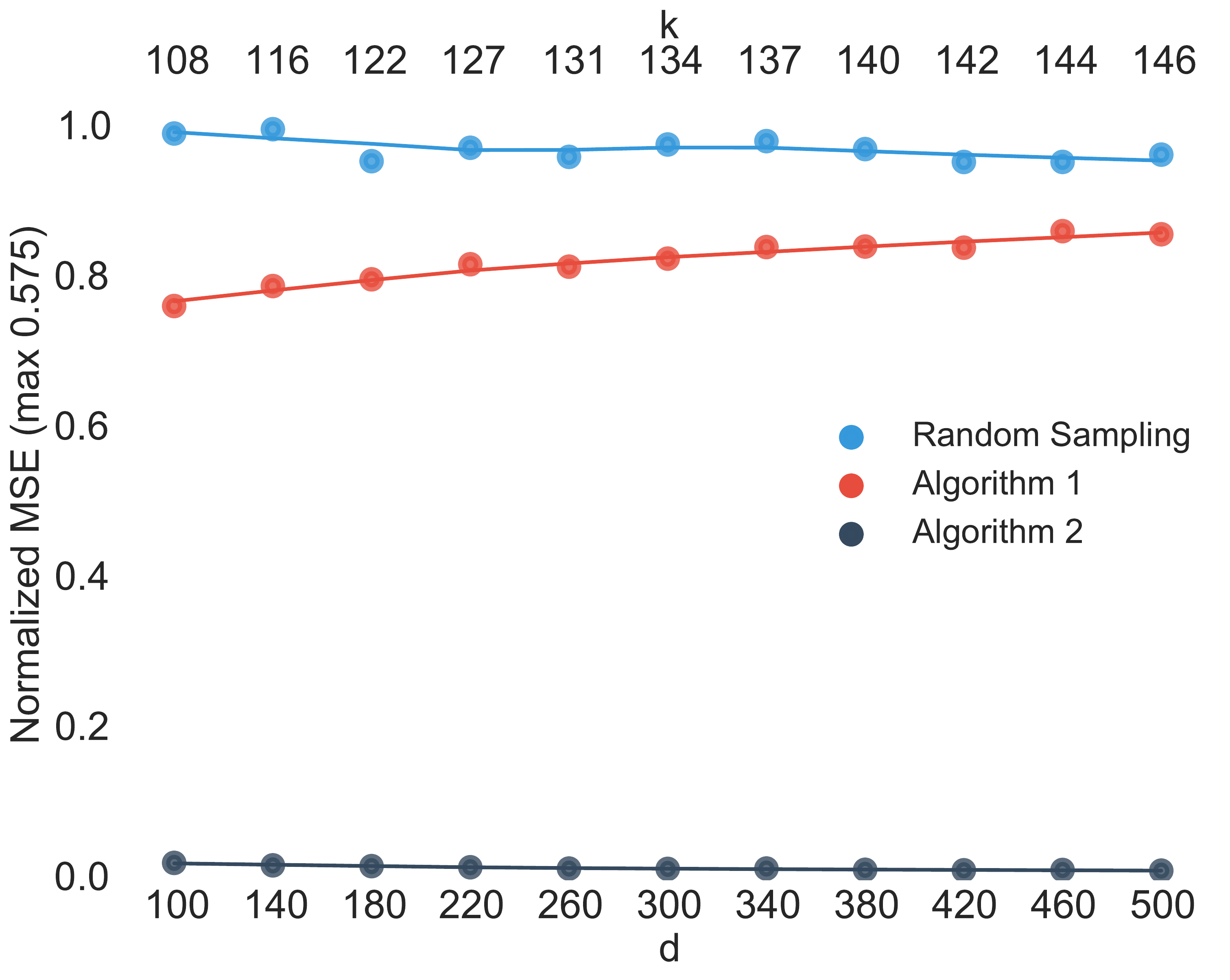}} 
  \subfigure[Zooming in.]{\includegraphics[width=0.49 \columnwidth]{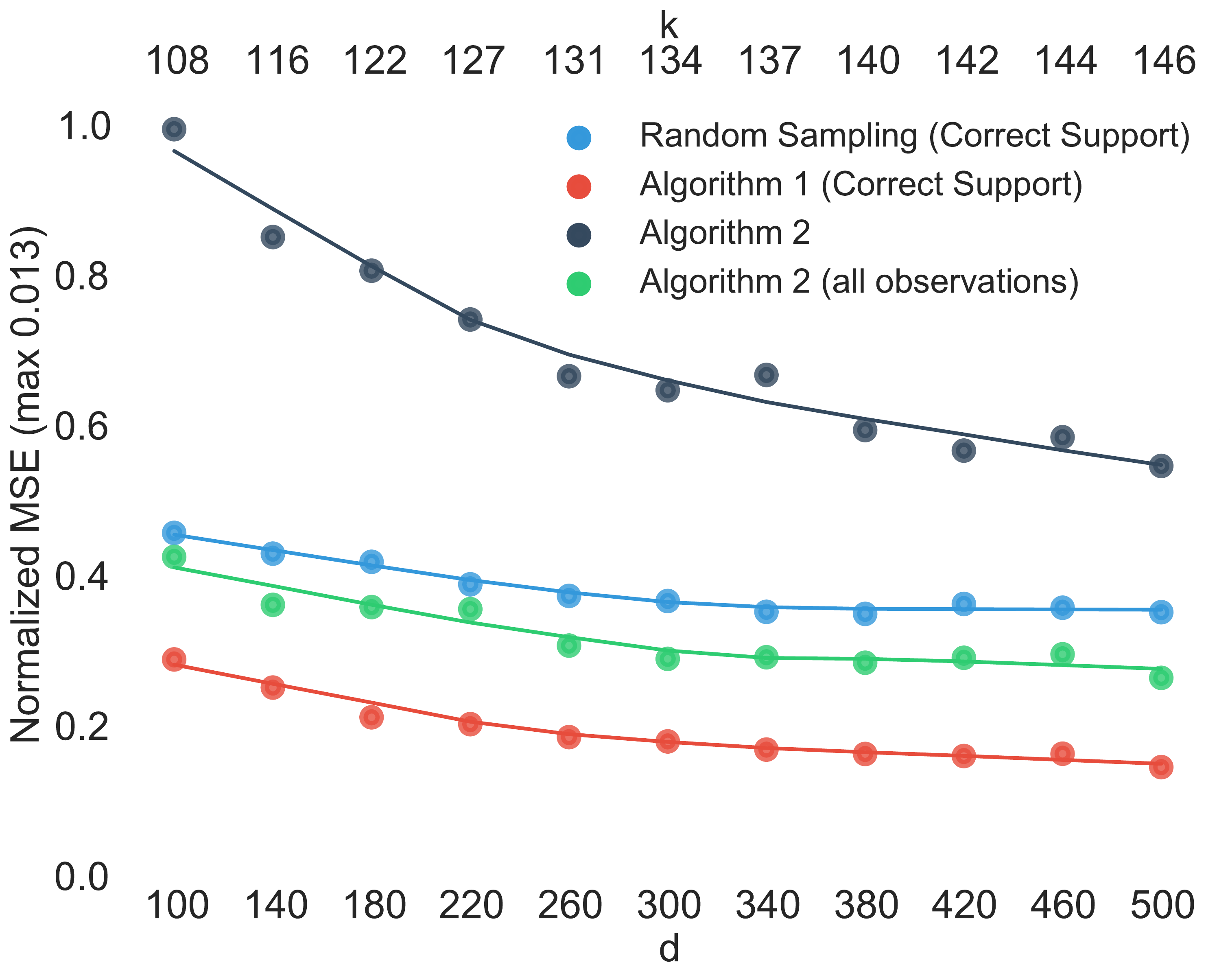}}
\caption{Sparse Linear Regression (700 iters). We fix the effective dimension to $s=7$, and increase the ambient dimension from $d=100$ to $d=500$. The budget scales as $k=C s \log d$ for $C \approx 3.4$, while $n=4d$. We set $k_1 = 2k/3$ and $k_2 = k/3$.}
 \label{lasso_plots}
\end{figure*}

\section{Lower Bound}
\label{lower}
In this section we derive a lower bound for the $k>d$ setting.
Suppose all the data are given. Again choose the $k$ observations with largest norms, denoted by $\X'$. To minimize the prediction error, 
 the best possible ${\X^{\prime}}^T\X^{\prime}$ is diagonal, with identical entries, and trace equal to the sum of the norms. No selection algorithm, online or offline, can do better.
Algorithm \ref{alg:threshold} achieves this by selecting observations with large norms and uncorrelated entries (through whitening if necessary).
Theorem \ref{th:lower} captures this intuition.
\begin{theorem}\label{th:lower}
Let $\mathbf{A}$ be an algorithm for the problem we described in Section 2.
Then,
\begin{align}\label{eq:lower_bound_th_white}
\E_{\mathbf{A}} \ \mathrm{Tr}(\Sigma (\X^T \X)^{-1}) &\ge \frac{d^2}{\E \left[ \sum_{i = 1}^k || X_{(i)}||^2 \right]} \\
&\ge \frac{d}{k \ \E \left[ \frac{1}{d} \ \max_{i \in [n]} || X_{i}||^2 \right]}, \nonumber
\end{align}
where $X_{(i)}$ is the white observation with the $i$-th largest norm.
Moreover, fix $\alpha \in (0, 1)$.
Let $\mathbf{F}$ be the cdf of $\max_{i \in [n]} ||X_{i}||^2$.
Then, $\mathrm{Tr}(\Sigma (\X^T \X)^{-1}) \ge {d^2}/{k \ \mathbf{F}^{-1}(1-\alpha)}$ with probability at least $1 - \alpha$. 
\end{theorem}
The proof is in Appendix E.
The upper bound in Theorem \ref{th:main} has a similar structure, with denominator equal to $k \phi$.
By Theorem \ref{th:main}, $\phi = \E_{\mathbf{D}} [X_{j}^2 \mid \| X \|_{\xi}^2 \ge \Gamma^2]$ for every component $j$.
Hence, summing over all components: $k \phi = k \E_{\bar {\mathbf{D}}} \left[ \| X \|^2 / d \right]$.
The latter expectation is taken with respect to $\bar{\mathbf{D}}$, which only captures the $k$ expected $\xi$-largest observations out of $n$, as opposed to $k \ \E_{\mathbf{D}} [ (1/k) \sum_{i = 1}^k || X_{(i)}||^2 / d ]$ in \eqref{eq:lower_bound_th_white}.
The weights $\xi$ simply account for the fact that, in reality, we cannot make all components have equal norm, something we implicitly assumed in our lower bound.

We specialize the lower bound to the Gaussian setting, for which we computed the upper bound of Theorem \ref{th:main}.
The proofs are based on the Fisher-Tippett Theorem and the Gumbel distribution; see Appendix F. 
\begin{corollary}
For Gaussian observations $X^i \sim \mathcal{N}(0, \Sigma)$ and large $n$, for any algorithm $\mathbf{A}$
\begin{equation*}
\E_{\mathbf{A}} \  \mathrm{Tr}(\Sigma (\X^T\X)^{-1}) \ge \frac{d}{k \left( \frac{2 \log n}{d} + \log \log n \right)}.
\end{equation*}
Moreover, let $\alpha \in (0, 1)$.
Then, for any $\mathbf{A}$ with probability at least $1 - \alpha$ and $C = 2 \log \Gamma(d/2) / d$,
\begin{equation*}
\mathrm{Tr}(\Sigma (\X^T\X)^{-1}) \ge \frac{d / k}{\frac{2 \log n}{d} + \log \log n - \frac{1}{d} \log \log\frac{1}{1 - \alpha} - C}
\end{equation*}
\end{corollary}
The results from Corollary \ref{co:gaussian} have the same structure as the lower bound; hence in this setting our algorithm is near optimal.
Similar results and conclusions are derived for the CLT approximation in Appendix I.

\section{Simulations}
\label{sim}

 \begin{figure*}[t!]
  \centering
  \subfigure[Protein Structure; 150 iters.]{\includegraphics[width=0.32 \columnwidth]{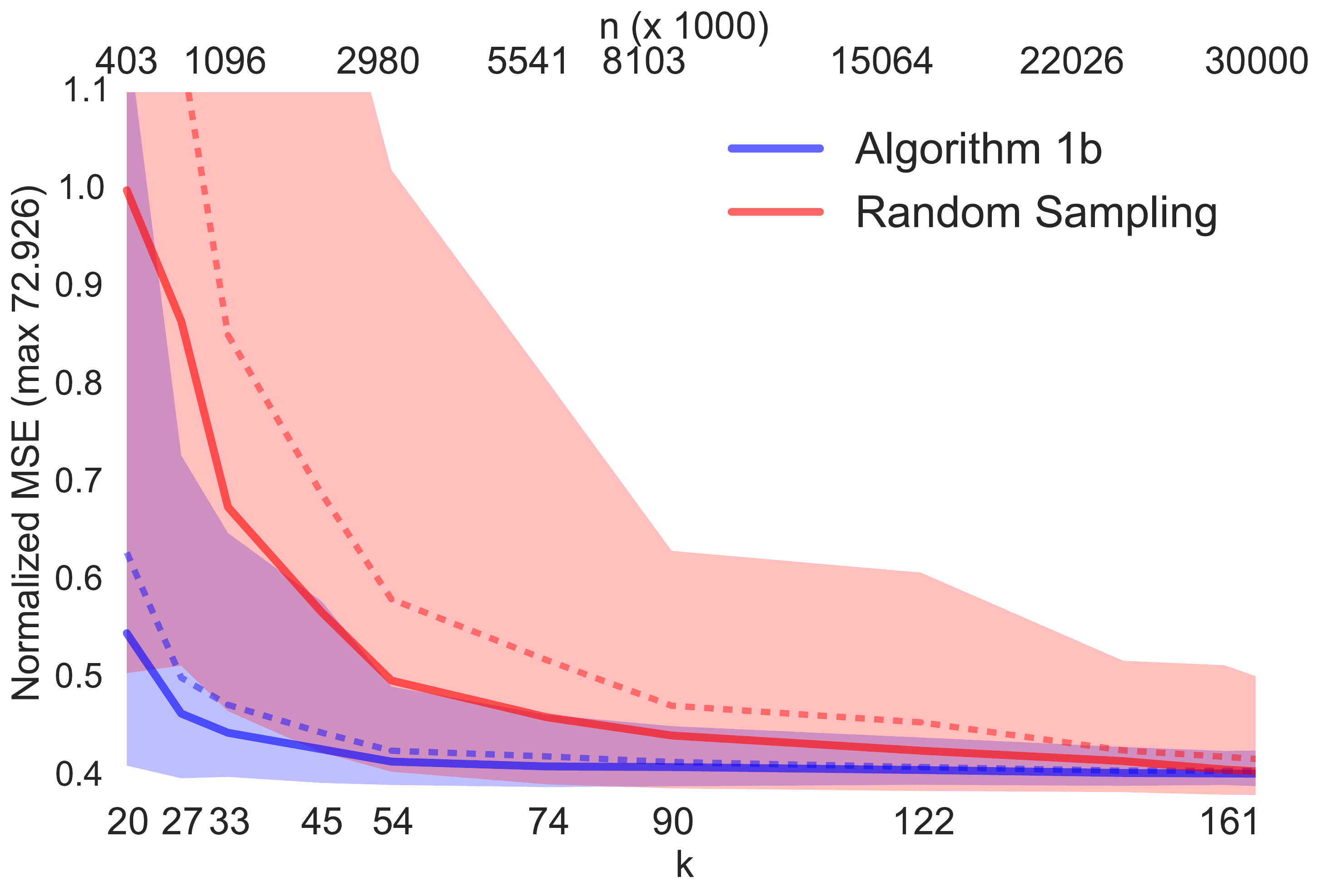}}\quad
  \subfigure[Bike Sharing; 300 iters.]{\includegraphics[width=0.32 \columnwidth]{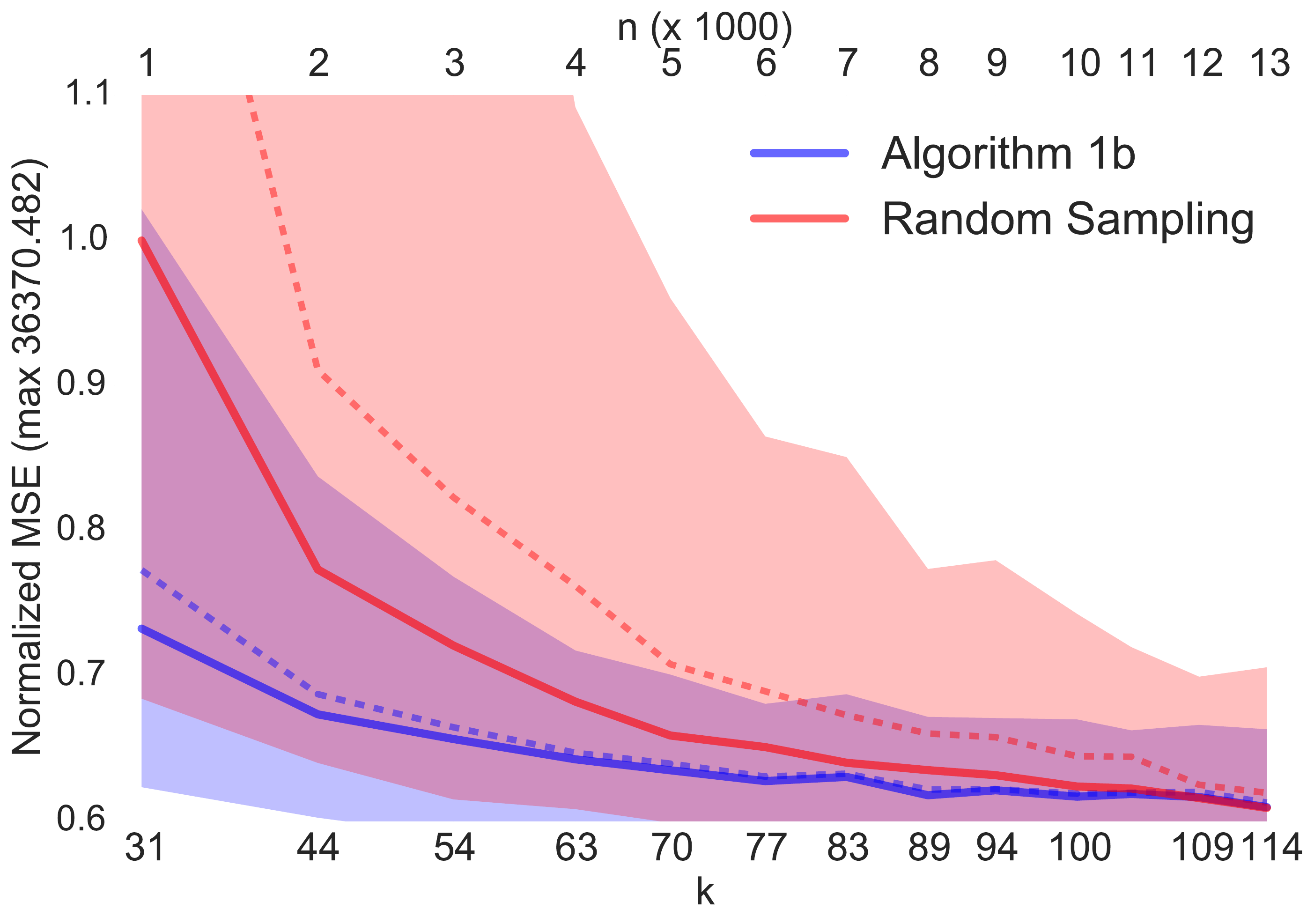}}
    \subfigure[YearPredictionMSD; 150 iters.]{\includegraphics[width=0.32 \columnwidth]{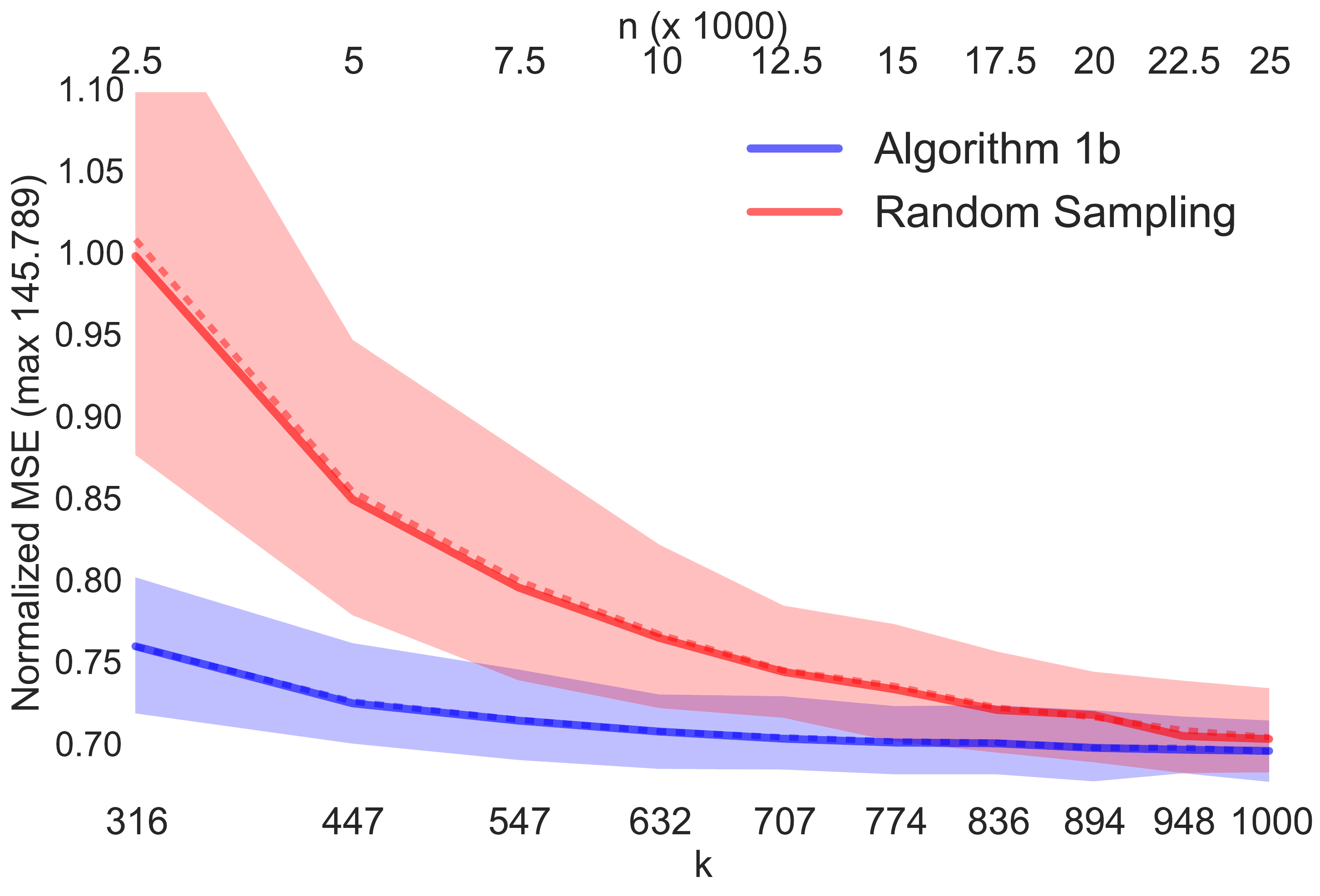}}
\caption{$\mathrm{MSE}$ of $\hat\beta_{OLS}$. The $(0.05, 0.95)$ quantile conf.\ int.\ displayed. Solid \emph{median}; Dashed \emph{mean}.}
 \label{realdata_plots}
\end{figure*}

We conducted experiments in various settings: regularized estimators in high-dimensions, and the basic thresholding approach in real-world data to explore its performance on strongly non-linear environments.

{\bf Regularized Estimators}.
We compare the performance in high-dimensional settings of random sampling and Algorithm \ref{alg:threshold} ---both with an appropriately adjusted Lasso estimator--- against Algorithm \ref{alg:threshold_lasso}, which takes into account the structure of the problem ($s \ll d$). 
For completeness, we also show the performance of Algorithm \ref{alg:threshold_lasso} when \emph{all} observations are included in the final OLS estimate, and that of random sampling (RS) and Algorithm \ref{alg:threshold} (Thr) when the true support $S$ is known in advance, and the OLS computed on $S$.
In Figure \ref{lasso_plots} (a), we see that Algorithm \ref{alg:threshold_lasso} dramatically reduces the MSE, while in Figure \ref{lasso_plots} (b) we zoom-in to see that, quite remarkably, Algorithm \ref{alg:threshold_lasso} using all observations for the final estimate outperforms random sampling that knows the sparsity pattern in hindsight.
We used $k_1 = (2/3) k$ for recovery.
More experiments are provided in Appendix K.

{\bf Real-World Data}.
We show the results of Algorithm 1b (online $\Sigma$ estimation)
with the simplest distributional assumption (Gaussian threshold, $\xi_j = 1$) versus random sampling on publicly available real-world datasets (UCI, \cite{Lichman:2013}), measuring test squared prediction error.
We fix a sequence of values of $n$, together with $k = \sqrt{n}$, and for each pair $(n, k)$ we run a number of iterations.
In each one, we randomly split the dataset in training ($n$ observations, random order), and test (rest of them).
Finally, $\hat\beta_{\mathrm{OLS}}$ is computed on selected observations, and the prediction error estimated on the test set.
All datasets are initially centered to have zero means (covariates and response).
Confidence intervals are provided.

We first analyze the Physicochemical Properties of Protein Tertiary Structure dataset (45730 observations), where we predict the size of the residue, based on $d = 9$ variables, including the total surface area of the protein and its molecular mass.
Figure \ref{realdata_plots} (a) shows the results; Algorithm \ref{alg:threshold}b outperforms random sampling for all values of $(n, k)$.
The reduction in variance is substantial.
In the Bike Sharing dataset \cite{bikedataset} we predict the number of hourly users of the service, given weather conditions, including temperature, wind speed, humidity, and temporal covariates.
There are 17379 observations, and we use $d = 12$ covariates.
Our estimator has lower mean, median and variance MSE than random sampling; Figure \ref{realdata_plots} (b).
Finally, for the YearPredictionMSD dataset \cite{songdataset}, we predict the year a song was released based on $d = 90$ covariates, mainly metadata and audio features.
There are 99799 observations.
The MSE and variance did strongly improve; Figure \ref{realdata_plots} (c).

In the examples we see that, while active learning leads to strong improvements in MSE and variance reduction for moderate values of $k$ with respect to $d$, the gain vanishes when $k$ grows large.
This was expected; the reason might be that by sampling so many outliers, we end up learning about parts of the space where heavy non-linearities arise, which may not be important to the test distribution. 
However, the motivation of active learning are situations of limited labeling budget, and
hybrid approaches combining random sampling and thresholding could be easily implemented if needed.

\section{Conclusion}
\label{conclusions}
Our paper provides a comprehensive analysis of thresholding algorithms for online active learning of linear regression models, which are shown to perform well both theoretically and empirically.  Several natural open directions suggest themselves.
Additional robustness could be guaranteed in other settings by combining our algorithm as a ``black box'' with other approaches: for example, some addition of random sampling or stratified sampling could be used to determine if significant nonlinearity is present, and to determine the fraction of observations that are collected via thresholding.

\section{Acknowledgments}
The authors would like to thank Sven Schmit for his excellent comments and suggestions, Mohammad Ghavamzadeh for fruitful discussions, and the anonymous reviewers for their valuable feedback.
We gratefully acknowledge support from the National Science Foundation under grants CMMI-1234955, CNS-1343253, and CNS-1544548.

\bibliography{online_active_learning}

\begin{thebibliography}{10}

\bibitem{balcan2006agnostic}
M.-F. Balcan, A.~Beygelzimer, and J.~Langford.
\newblock Agnostic active learning.
\newblock In {\em Proceedings of the 23rd international conference on Machine
  learning}, pages 65--72. ACM, 2006.

\bibitem{balcan2007margin}
M.-F. Balcan, A.~Broder, and T.~Zhang.
\newblock Margin based active learning.
\newblock In {\em Learning Theory}, pages 35--50. Springer, 2007.

\bibitem{balcan2010true}
M.-F. Balcan, S.~Hanneke, and J.~W. Vaughan.
\newblock The true sample complexity of active learning.
\newblock {\em Machine learning}, 80(2-3):111--139, 2010.

\bibitem{songdataset}
T.~Bertin-Mahieux, D.~P. Ellis, B.~Whitman, and P.~Lamere.
\newblock The million song dataset.
\newblock 2011.

\bibitem{cai2013maximizing}
W.~Cai, Y.~Zhang, and J.~Zhou.
\newblock Maximizing expected model change for active learning in regression.
\newblock In {\em Data Mining (ICDM), 2013 IEEE 13th International Conference
  on}, pages 51--60. IEEE, 2013.

\bibitem{castro2007minimax}
R.~M. Castro and R.~D. Nowak.
\newblock Minimax bounds for active learning.
\newblock pages 5--19, 2007.

\bibitem{cohn1994improving}
D.~Cohn, L.~Atlas, and R.~Ladner.
\newblock Improving generalization with active learning.
\newblock {\em Machine learning}, 15(2):201--221, 1994.

\bibitem{cohn1996active}
D.~A. Cohn, Z.~Ghahramani, and M.~I. Jordan.
\newblock Active learning with statistical models.
\newblock {\em Journal of artificial intelligence research}, 1996.

\bibitem{dasgupta2008hierarchical}
S.~Dasgupta and D.~Hsu.
\newblock Hierarchical sampling for active learning.
\newblock In {\em Proceedings of the 25th international conference on Machine
  learning}, pages 208--215. ACM, 2008.

\bibitem{dasgupta2007general}
S.~Dasgupta, C.~Monteleoni, and D.~J. Hsu.
\newblock A general agnostic active learning algorithm.
\newblock In {\em Advances in neural information processing systems}, pages
  353--360, 2007.

\bibitem{embrechts1997modelling}
P.~Embrechts, C.~Kl{\"u}ppelberg, and T.~Mikosch.
\newblock {\em Modelling extremal events}, volume~33.
\newblock Springer Science \& Business Media, 1997.

\bibitem{bikedataset}
H.~Fanaee-T and J.~Gama.
\newblock Event labeling combining ensemble detectors and background knowledge.
\newblock {\em Progress in Artificial Intelligence}, pages 1--15, 2013.

\bibitem{hoerl1970ridge}
A.~E. Hoerl and R.~W. Kennard.
\newblock Ridge regression: Biased estimation for nonorthogonal problems.
\newblock {\em Technometrics}, 12(1):55--67, 1970.

\bibitem{hsu2014heavy}
D.~Hsu and S.~Sabato.
\newblock Heavy-tailed regression with a generalized median-of-means.
\newblock In {\em Proceedings of the 31st International Conference on Machine
  Learning (ICML-14)}, pages 37--45, 2014.

\bibitem{inglot2010inequalities}
T.~Inglot.
\newblock Inequalities for quantiles of the chi-square distribution.
\newblock {\em Probability and Mathematical Statistics}, 30(2):339--351, 2010.

\bibitem{joseph2013variable}
A.~Joseph.
\newblock Variable selection in high-dimension with random designs and
  orthogonal matching pursuit.
\newblock {\em Journal of Machine Learning Research}, 14(1):1771--1800, 2013.

\bibitem{koltchinskii2010rademacher}
V.~Koltchinskii.
\newblock Rademacher complexities and bounding the excess risk in active
  learning.
\newblock {\em The Journal of Machine Learning Research}, 11:2457--2485, 2010.

\bibitem{krause2007nonmyopic}
A.~Krause and C.~Guestrin.
\newblock Nonmyopic active learning of gaussian processes: an
  exploration-exploitation approach.
\newblock In {\em Proceedings of the 24th international conference on Machine
  learning}, pages 449--456. ACM, 2007.

\bibitem{laurent2000adaptive}
B.~Laurent and P.~Massart.
\newblock Adaptive estimation of a quadratic functional by model selection.
\newblock {\em Annals of Statistics}, pages 1302--1338, 2000.

\bibitem{Lichman:2013}
M.~Lichman.
\newblock {UCI} machine learning repository.
\newblock 2013.

\bibitem{petersen2008matrix}
K.~B. Petersen et~al.
\newblock The matrix cookbook.

\bibitem{pukelsheim1993optimal}
F.~Pukelsheim.
\newblock {\em Optimal design of experiments}, volume~50.
\newblock siam, 1993.

\bibitem{sabato2014active}
S.~Sabato and R.~Munos.
\newblock Active regression by stratification.
\newblock In {\em Advances in Neural Information Processing Systems}, pages
  469--477, 2014.

\bibitem{sugiyama2009pool}
M.~Sugiyama and S.~Nakajima.
\newblock Pool-based active learning in approximate linear regression.
\newblock {\em Machine Learning}, 75(3):249--274, 2009.

\bibitem{tropp2005signal}
J.~Tropp and A.~C. Gilbert.
\newblock Signal recovery from partial information via orthogonal matching
  pursuit, 2005.

\bibitem{vershynin2010introduction}
R.~Vershynin.
\newblock Introduction to the non-asymptotic analysis of random matrices.
\newblock {\em arXiv preprint arXiv:1011.3027}, 2010.

\bibitem{wainwright2009sharp}
M.~J. Wainwright.
\newblock Sharp thresholds for high-dimensional and noisy sparsity recovery
  using-constrained quadratic programming (lasso).
\newblock {\em Information Theory, IEEE Transactions on}, 55(5):2183--2202,
  2009.

\bibitem{wang2014noise}
Y.~Wang and A.~Singh.
\newblock Noise-adaptive margin-based active learning and lower bounds under
  tsybakov noise condition.
\newblock {\em arXiv preprint arXiv:1406.5383}, 2014.

\end{thebibliography}
\bibliographystyle{abbrv}


\

{\Large \textbf{Appendix}}

\setcounter{section}{0}
\renewcommand{\thesection}{\Alph{section}}

\section{Whitening}
\label{whitening}
Before thresholding the norm of incoming observations, it is useful to decorrelate and standardize their components, i.e., to \emph{whiten} the data.
Then, we apply the algorithm to uncorrelated covariates, with zero mean and unit variance (not necessarily independent).
The covariance matrix $\Sigma$ can be decomposed as  $\Sigma = U D U^T$, where $U$ is orthogonal, and $D$ diagonal with $d_{ii} = \lambda_i(\Sigma)$.
We whiten each observation to $\bar{X} = D^{-1/2} U^T X \in \R^{d \times 1}$ (while for $\X \in \R^{k \times d}$, $\bar\X = \X U D^{-1/2}$), so that $\E \bar X \bar X^T = \mathrm{Id}$.
We denote whitened observations by $\bar X$ and $\bar \X$ in the appendix.
After some algebra we see that,
\begin{equation}\label{eq:trace_bound}
\frac{d}{\lambda_{\max}(\bar{\X}^T \bar{\X})} \le \mathrm{Tr}(\Sigma(\X^T\X)^{-1}) = \mathrm{Tr}((\bar{\X}^T \bar{\X})^{-1}) \le \frac{d}{\lambda_{\min}(\bar{\X}^T \bar{\X})}.
\end{equation}
We focus on algorithms that maximize the minimum eigenvalue of $\bar{\X}^T \bar{\X}$ with high probability, or, in general, leading to large and even eigenvalues of $\bar{\X}^T \bar{\X}$.

\section{Proof of Theorem 3.1}
\label{s:mainth}
\begin{theorem}\label{app_th:main}
Let $n > k > d$.
Assume observations $X \in \R^d$ are distributed according to subgaussian $\mathbf{D}$ with covariance matrix $\Sigma \in \R^{d \times d}$.
Also, assume marginal densities are symmetric around zero after whitening.
Let $\X$ be a $k \times d$ matrix with $k$ observations sampled from the distribution induced by the thresholding rule with parameters $(\xi, \Gamma) \in \R^{d+1}_+$ satisfying \eqref{eq:exp_threshold_cond}.
Let $\alpha > 0$, so that $t = \alpha \sqrt{k} - C \sqrt{d} > 0$, then, with probability at least $1 - 2 \exp(-c t^2)$
\begin{equation}\label{eq:th_conseq}
\mathrm{Tr}(\Sigma (\X^T\X)^{-1}) \le \frac{d}{(1 - \alpha)^2 \ \phi k},
\end{equation}
where constants $c, C$ depend on the subgaussian norm of $\bar{\mathbf{D}}$.
\end{theorem}
\begin{proof}
We would like to choose $k$ out of $n$ observations $X_1, \dots, X_n \sim \F$ iid.
Assume our sampling induces a new distribution $\bar \F$.
The loss we want to minimize for our OLS estimate $\hat\beta = \hat\beta(\X, \Y)$ is
\begin{align}
\E_{\X, \Y \sim \bar \F, X \sim \F} \left[ \left( X^T \hat\beta - X^T \beta \right)^2 \right] = \sigma^2 \ \E_{\X, \Y \sim \bar \F} \left[ \Tr \left( \Sigma (\X^T \X)^{-1} \right) \right],
\end{align}
where we assumed Gaussian noise with variance $\sigma^2$.

Let us see how we construct $\bar \F$.
We sample $X \sim \F$, we whiten the observation $Z = \Sigma^{-1/2} X$, and then we select it or not according to a fixed thresholding rule.
If $\| Z \|_{\xi} \ge \Gamma$, then we keep $X = \Sigma^{1/2} Z$.

We choose $\xi$ and $\Gamma$ so that there exists $\phi > 0$, such that for all $i = 1, \dots, d$,
\begin{equation}
\E_{W \sim \bar \F} [W_{(i)}^2] = \phi,
\end{equation}
where $W_{(i)}$ denotes the $i$-th component of $W \sim \bar \F$.
Note that $\bar \F = \bar \F(\xi, \Gamma)$.

\

$Z$ is a linear transformation of $X$; $W$ is \textbf{not} a linear transformation of $Z$.

Moreover, the covariance matrix of $\bar \F$ is $\Sigma_{\bar \F} = \phi \ \mathrm{Id}$.
If $\F$ is a general subgaussian distribution, thresholding could change the mean away from zero.

Assume after running our algorithm, we end up with $\X \in \R^{k \times d}$.
We denote by $\W$ the observations after whitening, note that by design every $w \in \W$ passed our test: $\| w \|_{\xi} \ge \Gamma$.
In other words, $w \sim \bar \F$.
We see that $\W = \X \Sigma^{-1/2}$ or, alternatively, $\X = \W \Sigma^{1/2}$.

Now, we can derive
\begin{align}
\Tr \left( \Sigma (\X^T \X)^{-1} \right) &= \Tr \left( \Sigma \left( \Sigma^{1/2} \W^T \W \Sigma^{1/2} \right)^{-1}  \right) \\
&= \Tr \left( \left( \Sigma^{-1/2} \Sigma^{1/2} \W^T \W \Sigma^{1/2} \Sigma^{-1/2} \right)^{-1}  \right) \\
&= \Tr \left( \left( \W^T \W \right)^{-1}  \right) \\
&= \Tr \left( \Sigma_{\bar \F}^{1/2} \Sigma_{\bar \F}^{-1} \Sigma_{\bar \F}^{1/2} \left( \W^T \W \right)^{-1}  \right) \\
&= \Tr \left( \Sigma_{\bar \F}^{-1} \left( \bar \W^T \bar \W \right)^{-1}  \right) \\
&= \Tr \left( \frac{1}{\phi} \mathrm{Id} \left( \bar \W^T \bar \W \right)^{-1}  \right) \\
&= \frac{1}{\phi} \ \Tr \left( \left( \bar \W^T \bar \W \right)^{-1}  \right) \\
&\le \frac{d}{\phi} \ \frac{1}{\lambda_{\min} \left(\bar \W^T \bar \W \right)}
= \frac{d}{\phi \ k} \ \frac{1}{\lambda_{\min} \left( \frac{1}{k} \bar \W^T \bar \W \right)}.
\end{align}
where $\bar \W$ is actually white data.
Thus, note that $\bar \W^T \bar \W / k \to \mathrm{Id}$ as $k \to \infty$.

\

Assume that $\F$ is subgaussian such that if $k > d$, then $\X^T\X$ has full rank with probability one.
Thresholding will not change the shape of the tails of the distribution, $\bar \F$ will also be subgaussian.

At this point, we need to measure how fast $\lambda_{\min} \left( \bar \W^T \bar \W \right) / k$ goes to 1.
We can use Theorem 5.39 in \cite{vershynin2010introduction} which guarantees that, for $\alpha > 0$ such that $t = \alpha \sqrt{k} - C \sqrt{d} > 0$, with probability at least $1 - 2\exp(-c t^2)$ we have
\begin{align}
\lambda_{\min} \left( \frac{1}{k} \bar \W^T \bar \W \right) \ge (1-\alpha)^2,
\end{align}
as $\bar\W$ is white subgaussian.
It follows that for $\alpha > C \sqrt{d} / \sqrt{k}$, with probability at least $1 - 2\exp(-c t^2)$
\begin{align}
\Tr \left( \Sigma (\X^T \X)^{-1} \right) \le \frac{d}{(1-\alpha)^2 \ \phi \ k}.
\end{align}
Note that $1/(1-\alpha) \approx 1 + O(\sqrt{{d}/{k}})$.
\end{proof}

\section{Proof of $\Tr(X^{-1}) \ge \Tr(\mathrm{Diag}(X)^{-1})$}
\label{s:le_trace}
In order to justify that we want $S = \X^T \X$ to be as close as possible to diagonal, we show the following lemma.
Under our assumptions $S$ is symmetric positive definite with probability 1.
\begin{lemma}
Let $X$ be a $n \times n$ symmetric positive definite matrix. Then,
\begin{equation}
\Tr(X^{-1}) \ge \Tr(\mathrm{Diag}(X)^{-1}),
\end{equation}
where $\mathrm{Diag}(\cdot)$ returns a diagonal matrix with the same diagonal as the argument.
\end{lemma}
In other words, we show that for all positive definite matrices with the same diagonal elements, the diagonal matrix (matrix with all off diagonal elements being 0) has the least trace after the inverse operation. 
\begin{proof}
We show this by induction. Consider a $2 \times 2$ matrix 
\begin{equation}
X=\begin{bmatrix}
a & b \\ b & c \end{bmatrix}
\end{equation}
and
\begin{equation}
\Tr(X^{-1})=\frac{1}{ac-b^2} (a+c)
\end{equation}
since $ac-b^2 >0$ ($X$ is positive definite), the above expression is minimized when $b^2=0$, that is, $X$ is diagonal.

\

Assume the statement is true for all $n \times n$ matrices. Let $X$ be a $(n+1) \times (n+1)$ positive definite matrix. Decompose it as
\begin{equation}
X=\begin{bmatrix}
A & b \\ b^T & c 
\end{bmatrix}.
\end{equation}
By the block inverse formula, (see for example \cite{petersen2008matrix})
\begin{equation}
\Tr(X^{-1})=\Tr(A^{-1})+\frac{1}{k}+\frac{1}{k} \Tr(A^{-1} b b^T A^{-1}),
\end{equation}
where $k=c-b^T A^{-1} b$.
Note $k >0$ by Schur's complement for positive definite matrices. Using the induction hypothesis, $\Tr(A^{-1})\geq \Tr(\mathrm{Diag}(A)^{-1})$.
By the positive definiteness of $A$, $b^T A^{-1} b \geq 0$, therefore $\frac{1}{k} \geq \frac{1}{c}$.

\

Also, $\Tr(A^{-1} b b^T A^{-1})\geq 0$.
Thus,
\begin{equation}
\Tr(X^{-1})\geq \Tr(A^{-1})+\frac{1}{c} = \Tr(\mathrm{Diag}(X)^{-1}),
\end{equation}
and the result follows.
\end{proof}

\section{Proof of Corollary 3.2}
\label{s:cogaussian}
\begin{corollary}\label{app_co:gaussian}
If the observations in Theorem \ref{th:main} are jointly Gaussian with covariance matrix $\Sigma \in \R^{d \times d}$, $\xi_j = 1$ for all $j = 1, \dots, d$, and $\Gamma = \bar{C} \sqrt{d + 2 \log(n / k)}$, for some constant $\bar{C}  \ge 1$, then with probability at least $1 - 2 \exp(-c t^2)$ we have that
\begin{equation}
\mathrm{Tr}(\Sigma(\X^T\X)^{-1}) \le \frac{d}{(1 - \alpha)^2 \left( 1 + \frac{2 \log(n/k)}{d} \right) \ k}.
\end{equation}
\end{corollary}
\begin{proof}
We have to show that $\xi_j = 1$ for all $j$, and $\Gamma = C \sqrt{d + 2 \log(n / k)}$ satisfy the equations
\begin{equation}\label{app_cor:prob_threshold_cond}
\p_{\mathbf{D}}\left( \| \bar X \|_{\xi} \ge \Gamma \right) = \alpha = \frac{k}{n},
\end{equation}
\begin{align}\label{app_cor:exp_threshold_cond}
\E_{D} [\bar X_{j}^2 \mid \| \bar X \|_{\xi}^2 \ge \Gamma^2] = \phi, \quad \text{ for all } j,
\end{align}
and $\phi > \left( 1 + {2 \log(n/k)}/{d} \right)$.
The components of $\bar X$ are independent, as observations are jointly Gaussian.
It immediately follows that $\xi_j = 1$, for all $1 \le j \le d$. Thus,
\begin{equation}
Z_{\xi} = \sum_{j=1}^d \bar X_j \sim \chi^2_d, \qquad \Gamma^2 = F^{-1}_{\chi^2_d}\left(1 - \frac{k}{n}\right).
\end{equation}
The value of $Z_{\xi}$ is strongly concentrated around its mean, $\E Z_{\xi}  = d$.
We now use two tail approximations to obtain our desired result.

By \cite{laurent2000adaptive}, we have that
\begin{equation}
\p(Z_{\xi} - d \ge 2 \sqrt{dx} + 2x) \le \exp(-x).
\end{equation}
If we take $\exp(-x) = \alpha$, then $x = \log(n/k)$.
In this case, we conclude that
\begin{equation}
\p\left(Z_{\xi} \ge d + 2 \log \left( \frac{n}{k} \right) + 2 \sqrt{d \log \left( \frac{n}{k} \right)} \right) \le \alpha = \frac{k}{n}.
\end{equation}
Note that $\p(\| \bar X \|_{\xi} > \Gamma) = \p(Z_{\xi} > \Gamma^2) = \alpha$.
Therefore, by definition
\begin{equation}
\Gamma \le \sqrt{d + 2 \log \left( \frac{n}{k} \right) + 2 \sqrt{d \log \left( \frac{n}{k} \right)}}.
\end{equation}

On the other hand, we would like to show that
\begin{equation}
\p\left(Z_{\xi} \ge d + 2 \log\left(\frac{n}{k}\right) \right) \ge \alpha,
\end{equation}
as that would directly imply that $\Gamma \ge \sqrt{d + 2 \log\left({n}/{k}\right)}$.

We can use Proposition 3.1 of \cite{inglot2010inequalities}.
For $d > 2$ and $x > d - 2$,
\begin{equation*}
\p(Z_{\xi} \ge x) \ge \frac{1-e^{-2}}{2} \frac{x}{x - d + 2\sqrt{d}} \exp \left\{ - \frac{1}{2} \left( x - d - (d-2) \log\left(\frac{x}{d}\right) + \log d \right) \right\}.
\end{equation*}
Take $x = d + 2 \psi$, where $\psi = \log(n/k)$.
It follows that
\begin{align}\label{app_eq:inglot}
\p(Z_{\xi} \ge d + 2\psi) &\ge \frac{1-e^{-2}}{2} \frac{d + 2 \psi}{2\sqrt{d} + 2 \psi} \exp \left\{ - \frac{1}{2} \left( 2 \psi - (d-2) \log\left(1 + \frac{2\psi}{d} \right) + \log d  \right) \right\} \nonumber \\
&= \frac{1-e^{-2}}{2} \frac{d + 2 \psi}{2\sqrt{d} + 2 \psi}  \exp \left\{ \frac{d-2}{2} \log\left(1 + \frac{2\psi}{d} \right) -\frac{1}{2} \log d \right\} \exp \left\{ - \psi \right\} \nonumber \\
&\ge \exp \left\{ - \psi \right\},
\end{align}
where we assumed, for example, $d \ge 9$ and $n / k > 17$ (as in Proposition 5.1 of \cite{inglot2010inequalities}).
In any case, in those rare cases (in our context) where $d < 9$ and $n / k$ very small, the previous bound still holds if we subtract a small constant $C \in [0, 5/2]$ from the LHS: $ \p(Z_{\xi} \ge d + 2\psi - C)$.

\

Equivalently, from \eqref{app_eq:inglot}
\begin{equation}
\p(Z_{\xi} \ge d + 2 \log(n/k)) \ge k/n = \alpha.
\end{equation}
We conclude that
\begin{equation}
\sqrt{d + 2 \log \left( \frac{n}{k} \right)} \le \Gamma \le \sqrt{d + 2 \log \left( \frac{n}{k} \right) + 2 \sqrt{d \log \left( \frac{n}{k} \right)}}.
\end{equation}
Finally, we have that
\begin{align}
\phi \ge \frac{\Gamma^2}{d} \ge 1 + \frac{2 \log \left( {n}/{k} \right)}{d}.
\end{align}
By Theorem \ref{app_th:main}, the corollary follows.
\end{proof}

\section{CLT Approximation}
\label{s:coclt}

As we explain in the main text, it is sometimes difficult to directly compute the distribution of the $\xi$-norm of a white observation, given by $Z_\xi$.
Recall that $\Gamma^2 = F^{-1}_{Z_\xi}(1 - k / n)$.
Fortunately, $Z_\xi$ is the sum of $d$ random variables, and, in high-dimensional spaces, a CLT approximation can help us to choose a good threshold.
In this section we derive some theoretical guarantees.

\

The CLT is a good idea for bounded variables (as the square is still bounded, and therefore subgaussian), but if the underlying components $X_{j}$ are unbounded subgaussian, $Z_\xi$ will be at least subexponential---as the square of a subgaussian random variable is subexponential, \cite{vershynin2010introduction}---, and a higher threshold ---like that coming from chi-squared--- is more appropriate.

\

In addition, in the context of heavy-tails, \emph{catastrophic} effects are expected, as $\p( \max_j X_j > t ) \sim \p( \sum_j X_j > t ) $, leading to observations dominated by single dimensions. 

\

Assume that components $\bar X_j$ are independent (while not necessarily identically distributed).
By Lyapunov's CLT, one can show that\footnote{Some mild additional moment/regularity conditions on each $\bar X_j$ are required to satisfy Lyapunov's Condition.}
\begin{equation*}
Z_\xi = \sum_{j=1}^d \xi_j \bar X_j^2 \approx \mathcal{N}\left(d, \sum_{j=1}^d \xi_j^2 \left( \E[\bar{X}_j^4] - 1 \right) \right).
\end{equation*}
It follows that $\Gamma$ satisfies $\p_{\mathbf{D}}\left( \| \bar X_i \|_{\xi} \ge \Gamma \right) = {k}/{n}$ if
\begin{equation*}
\Gamma^2 \approx d + \Phi^{-1} \left( 1 - \frac{k}{n} \right) \ \sqrt{\sum_{i=1}^d \xi_i^2 \left( \E[\bar X_i^4] - 1 \right)}.
\end{equation*}
In the sequel, assume $d$ is large enough, and the approximation error is negligible.

Define $\gamma = \sqrt{\sum_{i=1}^d \xi_i^2 \left( \E[\bar X_i^4] - 1 \right)}$.
\begin{corollary}\label{app_co:clt}
Assume $Z_\xi = \mathcal{N}\left(d, \gamma^2 \right)$ and $\Gamma^2 = d + \gamma \ \Phi^{-1} \left( 1 - {k}/{n} \right)$, with $\xi_j$ satisfying \eqref{eq:exp_threshold_cond}.
Let $\alpha > 0$, so that $t = \alpha \sqrt{k} - C \sqrt{d} > 0$. then with probability at least $1 - 2 \exp(-ct^2)$ we have that
\begin{equation}
\mathrm{Tr}(\Sigma(\X^T\X)^{-1}) \le \frac{d}{(1 - \alpha)^2 \ k \left( 1 + \frac{\gamma \sqrt{2 \log(n/k)}}{d} - O \left( \frac{\gamma \log\log(n/k) }{d \sqrt{\log(n/k)}} \right) \right)}.
\end{equation}
\end{corollary}
\begin{proof}
Note that, by definition, $\| X \|_{\xi}^2 \sim Z_\xi$ and $\Gamma$ jointly solve the equations required by Theorem \ref{app_th:main}.
In order to apply the theorem, all we need to do is to estimate the magnitude of 
\begin{equation}
\phi = \E_{D} [X_{j}^2 \mid \| X \|_{\xi}^2 \ge \Gamma^2] \ge \frac{\Gamma^2}{d} = 1 + \frac{\gamma}{d} \ \Phi^{-1} \left( 1 - {k}/{n} \right).
\end{equation}
Therefore, we want to find bounds on tail probabilities of the normal distribution.
By Theorem 2.1 of \cite{inglot2010inequalities}, we have that for small $k/n$
\begin{align}
\sqrt{2 \log (n/k)} - \frac{\log(4 \log(n/k)) + 2}{2 \sqrt{2 \log(n/k)}} &\le \Phi^{-1} \left( 1 - \frac{k}{n} \right) \\
&\le \sqrt{2 \log (n/k)} - \frac{\log(2 \log(n/k)) + 3/2}{2 \sqrt{2 \log(n/k)}},
\end{align}
and the result follows.
\end{proof}

We can also show how to apply the previous result to independent uniform distributions centered around zero.
In that case, we have that the fourth moment is $\E[\bar X_j^4] = 9/5$, so $\gamma = \sqrt{\frac{4}{5} d}$, leading to a gain factor
$$\phi = \left(1 + \sqrt{\frac{8 \log(n/k)}{5d}} - o \left( \frac{\log\log(n/k)}{\sqrt{d \log(n/k)}} \right)\right).$$

\section{Proof of Theorem 4.1}
\label{s:lw_theorem}
\begin{theorem}\label{app_th:lower}
Let $\mathbf{A}$ be an algorithm for the problem we described in Section 2.
Then,
\begin{align}\label{app_eq:lower_bound_th}
\E_{\mathbf{A}} \ \mathrm{Tr}(\Sigma (\X^T \X)^{-1}) &\ge \frac{d^2}{\E \left[ \sum_{i = 1}^k || \bar{X}_{(i)}||^2 \right]} \\
&\ge \frac{d}{k \ \E \left[ \frac{1}{d} \ \max_{i \in [n]} || \bar{X}_{i}||^2 \right]}, \nonumber
\end{align}
where $\bar{X}_{(i)}$ is the white observation with the $i$-th largest norm.
Moreover, fix $\alpha \in (0, 1)$.
Let $\mathbf{F}$ be the cdf of $\max_{i \in [n]} ||X_{i}||^2$.
Then, with probability at least $1 - \alpha$
\begin{equation}\label{app_th:lower_hp}
\mathrm{Tr}(\Sigma (\X^T \X)^{-1}) \ge {d^2}/{k \ \mathbf{F}^{-1}(1-\alpha)}.
\end{equation} 
\end{theorem}
\begin{proof}
We want to minimize $\mathrm{Tr}(\Sigma (\X^T \X)^{-1}) = \mathrm{Tr}((\bar \X^T \bar \X)^{-1})$.
Let us define $S = \bar \X^T \bar \X$.
One can prove that $H \to \text{Tr}(H^{-1})$ is \emph{convex} for symmetric positive definite matrices $H$.
It then follows by Jensen's Inequality (assuming $k > d$, so $S$ is symmetric positive definite with high probability)
\begin{equation}\label{app_eq:lower_bound_tr_inv}
\E  \text{Tr}(S^{-1}) \ge  \text{Tr}((\E S)^{-1}) = \sum_{j=1}^d \frac{1}{\lambda_j(\E S)}.
\end{equation}
Let $\E S$ be the expected value of $S$ for an \emph{arbitrary} algorithm $\mathbf{A}$ that selects its observations sequentially.
We want to understand what is the \emph{minimum} possible value the RHS of \eqref{app_eq:lower_bound_tr_inv} can take.
The sum of eigenvalues is upper bounded by
\begin{align*}
\sum_{j=1}^d \lambda_j(\E S) = \text{Tr}(\E S) = \sum_{j=1}^d \E(S_{jj}) &= \sum_{j=1}^d \sum_{i=1}^k \E[\bar X_{ij}^2] \\
&= {\sum_{i=1}^k \E[||\bar X_{i}||^2]} \\
&\le {\E \left[ \sum_{i = 1}^k || \bar X_{(i)}||^2 \right]} \\
&\le {k \ \E \left[ \max_{i \in [n]} || \bar X_{i}||^2 \right]},
\end{align*}
where $\bar X_{(i)}$ denotes the observation with the $i$-th largest norm.
Because $\E S$ is symmetric positive definite, its eigenvalues are real non-negative, so that
\begin{align*}
0 < \lambda_{\min}(\E S) \le \frac{\mathrm{Tr}(\E S)}{d} \le \frac{{\E \left[ \sum_{i = 1}^k ||\bar X_{(i)}||^2 \right]}}{d} \le \frac{k \ \E \left[ \max_{i \in [n]} ||\bar {X}_{i}||^2 \right]}{d}.
\end{align*}
We conclude that the \emph{solution} to the minimization problem of \eqref{app_eq:lower_bound_tr_inv} ---that is, when all eigenvalues are equal--- is lower bounded by
\begin{equation*}
\E \mathrm{Tr}(S^{-1}) \ge \sum_{j=1}^d \frac{1}{\lambda_j(\E S)} \ge \frac{d^2}{{\E \left[ \sum_{i = 1}^k || \bar X_{(i)}||^2 \right]}} \ge \frac{d^2}{k \ \E \left[ \max_{i \in [n]} ||\bar {X}_{i}||^2 \right]},
\end{equation*}
which proves \eqref{app_eq:lower_bound_th}.

\

In order to prove the high-probability statement \eqref{app_th:lower_hp}, note that
\begin{align}
\mathrm{Tr}(\Sigma (\X^T\X)^{-1}) = \mathrm{Tr}((\bar \X^T\bar \X)^{-1}) &= \sum_{i = 1}^d \frac{1}{\lambda_{i} (\bar \X^T\bar \X)} \nonumber \\
&\ge \sum_{i = 1}^d \frac{1}{\sum_{j = 1}^k \| \bar X_j \| ^2 / d} \nonumber \\
&\ge \frac{d^2}{\sum_{j = 1}^k \| \bar X_{(j)} \| ^2} \ge \frac{d^2}{k \max_{i \in [n]} \| \bar X_{i} \| ^2}.
\end{align}
We directly conclude that with probability at least $1 - \alpha$,
\begin{equation}
\max_{i \in [n]} \| \bar X_{i} \| ^2 \le \mathbf{F}^{-1}(1-\alpha)
\end{equation}
as $\mathbf{F}$ is the cdf of $\max_{i \in [n]} \| \bar X_{i} \| ^2$.
It follows that with probability at least $1 - \alpha$,
\begin{equation}
\mathrm{Tr}(\Sigma (\X^T\X)^{-1}) \ge \frac{d^2}{k \ \mathbf{F}^{-1}(1-\alpha)}.
\end{equation}
\end{proof}

\section{Proof of Corollary 4.2}
\label{s:lw_guassian}
\begin{corollary}
For Gaussian observations $X_i \sim \mathcal{N}(0, \Sigma)$ and large $n$, for any algorithm $\mathbf{A}$
\begin{equation}\label{eq_exp_gauss_lwbound}
\E_{\mathbf{A}} \  \mathrm{Tr}(\Sigma (\X^T\X)^{-1}) \ge \frac{d}{k \left( \frac{2 \log n}{d} + \log \log n \right)}.
\end{equation}
Moreover, let $\alpha \in (0, 1)$.
Then, for any $\mathbf{A}$ with probability at least $1 - \alpha$ and $C = 2 \log \Gamma(d/2) / d$,
\begin{equation}\label{eq_hp_gauss_lwbound}
\mathrm{Tr}(\Sigma (\X^T\X)^{-1}) \ge \frac{d}{k \left( \frac{2 \log n}{d} + \log \log n - \frac{1}{d} \log \log\frac{1}{1 - \alpha} - C \right)}.
\end{equation}
\end{corollary}

\begin{proof}
In order to apply Theorem \ref{app_th:lower}, we need to upper bound $\E \left[ \max_{i \in [n]} ||\bar X_{i}||^2 \right]$, where $\bar X_{i}$ is a $d$-dimensional gaussian random variable with identity covariance matrix.
In other words, we need to upper bound the expected maximum of $n$ chi-squared random variables with $d$ degrees of freedom.

\

Let us start by proving \eqref{eq_exp_gauss_lwbound}.
We can use \emph{extreme value theory} to find the limiting distribution of the maximum of $n$ random variables.
Firstly, note that the chi-squared distribution is a particular case of the Gamma distribution.
More specifically, $\chi^2_d \sim \Gamma(d/2, 2)$.
If we parameterize the $\Gamma$ distribution by $\alpha$ (shape) and $\beta$ (rate), then $\alpha = d/2$ and $\beta = 1/2$.

\

By the \emph{Fisher-Tippett Theorem} we know that there are only \emph{three} limiting distributions for $\lim_{n \to \infty}  X_{(n)} = \lim_{n \to \infty} \max_{i \le n} X_i$, where the $X_i$ are iid random variables, namely, Frechet, Weibull and Gumbel distributions.
It is known that the Gamma distribution is in the max-domain of attraction of the Gumbel distribution.
Further, the normalizing constants are known (see Chapter 3 of \cite{embrechts1997modelling}).
In particular, we know that if $X_{(n)} := \max_{i \in [n]} ||\bar X_{i}||^2$
\begin{equation}\label{eq:asymptotic_extreme_value}
\lim_{n \to \infty} \p \left( X_{(n)} \le 2x + 2 \ln n + 2 (d/2 - 1) \ln \ln n - 2 \ln \Gamma(d/2) \right) = \Lambda(x) = e^{-e^{-x}}.
\end{equation}
We can \emph{assume} that the asymptotic limit holds, as $n$ is in practice very large, and compute the mean value of $X_{(n)}$.
As $X_{(n)}$ is a positive random variable,
\begin{align}
\E[X_{(n)}] &= \int_{0}^{\infty} \p \left( X_{(n)} \ge t \right) \ dt \\
&= \int_{0}^{\infty} (1 - \p \left( X_{(n)} \le t \right) ) \ dt
\end{align}
We make the change of variables $t = 2x + C$, where $C = 2 \ln n + (d - 2) \ln \ln n - 2 \ln \Gamma(d/2)$.
Then,
\begin{align}
\E[X_{(n)}] &= \int_{0}^{\infty} \p \left( X_{(n)} \ge t \right) \ dt \\
&= \int_{-C/2}^{\infty} 2 (1 - \p \left( X_{(n)} \le 2x + C \right) ) \ dx \\
&\approx \int_{-C/2}^{\infty} 2 (1 - e^{-e^{-x}}) \ dx \\
&= \int_{-C/2}^{0} 2 (1 - e^{-e^{-x}}) \ dx + \int_{0}^{\infty} 2 (1 - e^{-e^{-x}}) \ dx \\
&\le \int_{-C/2}^{0} 2 \ dx + 2 \gamma = C + 2 \gamma,
\end{align}
where $\gamma$ is the Euler--Mascheroni constant. We conclude that
\begin{equation}
\E[X_{(n)}] \le C + 2 \gamma \le 2 \ln n + (d - 2) \ln \ln n.
\end{equation}
If we take the largest $k$ observations, and assume we could split the weight equally among all dimensions (which is desirable), we see that the best we can do in expectation is upper bounded by
\begin{equation}
\frac{k}{d} \ \E[X_{(n)}] \le k \left( \frac{2 \ln n}{d} + \ln \ln n \right).
\end{equation}

Now, let us prove \eqref{eq_hp_gauss_lwbound}.
The following inequalities simplify our task to finding a high-probability upper bound on $\max_{i \in [n]} \| \bar X_i \|^2$.
We have that
\begin{align}
\mathrm{Tr}(\Sigma (\X^T\X)^{-1}) = \mathrm{Tr}((\bar \X^T\bar \X)^{-1}) &= \sum_{i = 1}^d \frac{1}{\lambda_{i} (\bar \X^T\bar \X)} \nonumber \\
&\ge \sum_{i = 1}^d \frac{1}{\sum_{j = 1}^k \| \bar X_j \| ^2 / d} \nonumber \\
&\ge \frac{d^2}{\sum_{j = 1}^k \| \bar X_{(j)} \| ^2} \ge \frac{d^2}{k \max_{i \in [n]} \| \bar X_{i} \| ^2}.
\end{align}
Fix $\alpha \in [0,1]$.
We need to find a constant $Q$ such that with probability at least $1 - \alpha$, $Q \ge \max_{i \in [n]} \| \bar X_{i} \| ^2$, so that we conclude that $\mathrm{Tr}(\Sigma (\X^T\X)^{-1}) \ge d^2 / Qk$ with high probability.
By \eqref{eq:asymptotic_extreme_value} we know that
\begin{equation}
\lim_{n \to \infty} \p \left( X_{(n)} \le 2x + 2 \ln n + 2 (d/2 - 1) \ln \ln n - 2 \ln \Gamma(d/2) \right) = \Lambda(x) = e^{-e^{-x}}.
\end{equation}
For large $n$, we assume the previous upper bound for $X_{(n)}$ is exact.
We want to find $Q > 0$ such that $\p \left( X_{(n)} \le Q \right) = 1 - \alpha$.
Note that if $1 - \alpha = e^{-e^{-x}}$, then
\begin{equation}
x = - \log \log \frac{1}{1 - \alpha}.
\end{equation}
It follows that $Q = 2 \ln n + 2 (d/2 - 1) \ln \ln n - \log \log (1 - \alpha)^{-1} - 2 \ln \Gamma(d/2)$.
Finally, \eqref{eq_hp_gauss_lwbound} follows as
\begin{equation}
\frac{Q}{d} = \frac{2 \log n}{d} + \log \log n - \frac{\log \log (1 - \alpha)^{-1} + 2 \ln \Gamma(d/2)}{d}.
\end{equation}
\end{proof}

\section{Proof of Theorem 3.3}
\label{s:ridge}
Recall the Sparse Thresholding Algorithm below.
\begin{algorithm}[ht]
\begin{algorithmic}[1]
\STATE Set $S_1 = \emptyset, S_2 = \emptyset$. Let $k = k_1 + k_2, n = k_1 + n_2$.
\FOR{observation $1 \le i \le k_1$}
\STATE Observe $X^i$. Choose $X^i$: $S_1 = S_1 \cup X^i$.
\ENDFOR
\STATE Set $\gamma = 1/2, \lambda = \sqrt{4 \sigma^2 \log(d) / \gamma^2 k_1}$.
\STATE Compute Lasso estimate $\hat\beta_1$ based on $S_1$, with regularization $\lambda$.
\STATE Set weights: $\xi_i = 1$ if $i \in S(\hat\beta_1)$, $\xi_i = 0$ otherwise.
\STATE Set $\Gamma = C \sqrt{s + 2 \log(n_2 / k_2)}$. Factorize $\Sigma_{S(\hat\beta_1)S(\hat\beta_1)} = U D U^T$.
\FOR{observation $k_1 + 1 \le i \le n$}
\STATE Observe $X^i \in \R^d$. Restrict to $X^i_S := X^i_{S(\hat\beta_1)} \in \R^s$.
\STATE Compute $\overline{X^i}_S = {D}^{-1/2} U^T X^i_S$.
\IF{$\| \overline{X^i_S} \|_{\xi} > \Gamma$ or $k_2 - |S_2| = n - i + 1$}
\STATE Choose $X^i_S$: $S_2 = S_2 \cup X^i_S$.
\IF{$|S_2| = k_2$}
\STATE \textbf{break}.
\ENDIF
\ENDIF
\ENDFOR
\STATE Return OLS estimate $\hat\beta_2$ with observations in $S_2$.
\end{algorithmic}
\caption{Sparse Thresholding Algorithm.}
\label{app_alg:lasso_threshold}
\end{algorithm}
We show the following theorem.
\begin{theorem}
\label{app_th:lasso_algo}
Let $\mathbf{D} = \mathcal{N}(0, \Sigma)$.
Assume $\Sigma, \lambda$ and $\min_i |\beta_i|$ satisfy the standard conditions given in Theorem 3 of \cite{wainwright2009sharp}.
Assume we run the Sparse Thresholding algorithm with $k_1 = C^{\prime} s \log d$ observations to recover the support of $\beta$, for an appropriate $C^{\prime} \ge 0$.
Let $\X_2$ be $k_2 = k - k_1$ observations sampled via thresholding on $S(\hat\beta_1)$.
It follows that for $\alpha > 0$ such that $t = \alpha \sqrt{k_2} - C \sqrt{s} > 0$, there exist some universal constants $c_1, c_2$, and $c, C$ that depend on the subgaussian norm of $\bar{\mathbf{D}} \mid S(\hat{\beta}_1)$, such that with probability at least
$$1 - 2e^{- \min \left( c_2 \min(s, \log(d-s)) - \log(c_1), ct^2 - \log(2) \right)}$$
it holds that 
\begin{equation*}\label{th:lasso_algo_perf}
\mathrm{Tr}(\Sigma_{SS}(\X_2^T\X_2)^{-1}) \le \frac{s}{(1 - \alpha)^2 \left( 1 + \frac{2 \log\left({n_2}/{k_2}\right)}{s} \right) \ k_2}.
\end{equation*}
\end{theorem}

For support recovery, we use Theorem 3 from \cite{wainwright2009sharp}:
\begin{theorem}
Consider the linear model with random Gaussian design
\begin{equation}
Y = \X \beta^* + \epsilon, \qquad \text{ with } k \text{ i.i.d.\ rows } x_i \sim \mathcal{N}(0, \Sigma) \in \R^d,
\end{equation}
with noise $\epsilon \sim \mathcal{N}(0, \sigma^2 \ \mathrm{Id}_{k \times k})$.
Assume the covariance matrix $\Sigma$ satisfies
\begin{equation}\label{eq:mutual_inc}
\| \Sigma_{S^CS}(\Sigma_{SS})^{-1} \|_{\infty} \le (1-\gamma), \quad \text{ for some } \gamma \in (0,1],
\end{equation}
\begin{equation}\label{eq:min_eig}
\lambda_{\min}(\Sigma_{SS}) \ge C_{\min} > 0.
\end{equation}
Let $|S| = s$.
Consider the family of regularization parameters for $\phi_d \ge 2$
\begin{equation}\label{eq:regularization}
\lambda_k(\phi_d) = \sqrt{\frac{\phi_d \ \rho_u(\Sigma_{S^CS})}{\gamma^2} \frac{2 \sigma^2 \log(d)}{k}}.
\end{equation}
If for some fixed $\delta > 0$, the sequence $(k, d, s)$ and regularization sequence $\{ \lambda_k \}$ satisfy
\begin{equation}
\frac{k}{2s \log(d-s)} \ge (1+\delta) \ \theta_u(\Sigma) \left( 1 + \frac{\sigma^2 C_{\min}}{\lambda_k^2 s} \right),
\end{equation}
then the following holds with prob at least $1 - c_1 \exp(- c_2 \min \{ s, \log(d-s) \} )$:
\begin{enumerate}
\item The Lasso has a unique solution $\hat\beta$ with support in $S$ (i.e.\, $S(\hat\beta) \subset S(\beta^*)$).
\item Define the gap
\begin{equation}\label{eq:min_beta}
g(\lambda_k) := c_3 \lambda_k \| \Sigma_{SS}^{-1/2} \|_\infty^2 + 20 \sqrt{\frac{\sigma^2 \log s}{C_{\min} \ k}}.
\end{equation}
Then, if $\beta_{\min} := \min_{i \in S} |\beta_i^*| > g(\lambda_k)$, the signed support $S_{\pm}(\hat\beta)$ is identical to $S_{\pm}(\beta^*)$, and moreover $\| \hat\beta_S - \beta^*_S \|_{\infty} \le g(\lambda_k)$.
\end{enumerate}
\end{theorem}
The required definitions to apply the previous theorem are
\begin{align}
\rho_l(\Sigma) &= \frac{1}{2} \min_{i \neq j} \left( \Sigma_{ii} + \Sigma_{jj} - 2 \Sigma_{ij} \right), &\quad \rho_u(\Sigma) = \max_{i} \Sigma_{ii}, \\
\theta_l(\Sigma) &= \frac{\rho_l(\Sigma_{S^C|S})}{C_{\max} (2 - \gamma(\Sigma))^2}, &\quad
\theta_l(\Sigma) = \frac{\rho_u(\Sigma_{S^C|S})}{C_{\min} \gamma^2(\Sigma)}.
\end{align}

\
\begin{proof}{\emph{(Theorem H1)}}

Let $X \sim \mathcal{N}(0, \Sigma)$ with $\Sigma$ satisfying \eqref{eq:mutual_inc} and \eqref{eq:min_eig}.
Let $\lambda_k(\phi_d)$ be like in \eqref{eq:regularization}, for some $\phi_d > 2$.
Assume we choose the number of observations $k_1$ in the first stage to be at least
\begin{align}
k_1 &\ge 2 (1+\delta) \ \theta_u(\Sigma) \left( 1 + \frac{\sigma^2 C_{\min}}{\lambda_k^2 s} \right) s \log(d-s) \\
&= C(\Sigma, d, s) \ s \log(d-s),
\end{align}
and that $\beta_{\min}$ is greater than \eqref{eq:min_beta}.
Then, with probability at least
$$1 - c_1 \exp(- c_2 \min \{ s, \log(d-s) \} ),$$ we recover the right support $S(\beta^*) = S(\hat\beta)$ in the first stage of the algorithm.

\

Conditional on this event, we apply our algorithm on the remaining observations.
In the second stage, we only look at those dimensions in $S(\hat\beta)$, by setting weights $\xi_{S(\hat\beta)} = 1$, and zero otherwise.
Finally, we run OLS along the dimensions in the recovered support, and using the observations collected during the second stage.
Importantly, note that the new observations are $\mathcal{N}(0, \Sigma_{SS})$.

\

We can now apply our original results.
Denote by $\X_2 \in \R^{k_2 \times s}$ the set of observations collected in the second stage of the algorithm.

In particular, by Corollary \ref{app_co:gaussian}, we conclude that for $\alpha > 0$ such that $t = \alpha \sqrt{k_2} - C \sqrt{s} > 0$, the following holds with probability at least $1 - 2 \exp(-ct^2)$
\begin{equation}\label{eq_final_sparse_guarantee}
\mathrm{Tr}(\Sigma_{SS}(\X_2^T\X_2)^{-1}) \le \frac{s}{(1 - \alpha)^2 \left( 1 + \frac{2 \log\left({n_2}/{k_2}\right)}{s} \right) \ k_2}.
\end{equation}

Under the event that the recovery is correct, the contribution to the MSE of the components of $\beta$ that are not in its support is zero.
In other words,
\begin{align}
\| \beta - \hat\beta_2 \|_{\Sigma}^2 &= (\beta - \hat\beta_2)^T \Sigma (\beta - \hat\beta_2) \\
&= (\beta_S - \hat\beta_{2S})^T \Sigma_{SS} (\beta_S - \hat\beta_{2S}) = \| \beta_S - \hat\beta_{2S} \|_{\Sigma_{SS}}^2.
\end{align}

As the events that the first and second stages succeed are independent, we conclude \eqref{eq_final_sparse_guarantee} holds with probability at least
\begin{align}
1 - c_1 e^{- c_2 \min \{ s, \log(d-s) \}} &- 2 e^{-c t^2} \ge \\
&1 - 2 e^{- \min \left( c_2 \min(s, \log(d-s)) - \log(c_1), ct^2 - \log(2) \right)}.
\end{align}
\end{proof}

\section{Proof of CLT Lower Bound}
\label{s:lw_clt}
\begin{corollary}
Assume the norm of white observations is distributed according to $Z_\xi = \mathcal{N}\left(d, \gamma^2 \right)$.
Then, we have that for any algorithm $\mathbf{A}$
\begin{equation}
\E_{\mathbf{A}} \  \mathrm{Tr}(\Sigma(\X^T\X)^{-1}) \ge \frac{d}{\left( 1 + \frac{\gamma}{d} \ \sqrt{2 \log n} \right) k }.
\end{equation}
\end{corollary}
\begin{proof}
By Theorem \ref{app_th:lower}, we need to compute $\E \left[ \max_{i \in [n]} ||X_{i}||^2 \right]$.

By assumption $\| X_{i} \|^2 \sim \mathcal{N}\left(d, \gamma^2 \right)$ for each $i$, which implies
\begin{align}
\E \left[ \max_{i \in [n]} ||X_{i}||^2 \right] &= \E \left[ d + \max_{i \in [n]} {\gamma} \ \frac{ ||X_{i}||^2 - d }{\gamma} \right] \\
&\le d + \gamma \ \E \left[ \max_{i \in [n]} \mathcal{N}(0, 1) \right] \\
&\le d + \gamma \sqrt{2 \log n},
\end{align}
and the result follows.
\end{proof}

\section{Ridge Regression}
\label{s:ridge}
Regularized linear estimators also benefit from large and balanced observations.
We show that, under mild assumptions, the performance of the ridge regression is directly aligned with that of previous sections.

\

The ridge estimator is $\hat\beta_\lambda = \left( \X^T \X + \lambda I \right)^{-1} \ \X^T \y$, given $(\X, \y)$ and $\lambda > 0$.
The following result shows how large values of $\lambda_{\min}(\X^T\X)$ help to control the MSE of $\hat\beta_\lambda$.
As the optimal penalty parameter $\lambda^*$ is unknown until the end of the data collection process, we assume it is \emph{uniformly} random in a small interval.
\begin{theorem}
Let $R > 0$. Assume the penalty parameter for ridge regression is chosen uniformly at random $\lambda^* \sim U[0, R]$.
Then, the MSE of $\hat\beta_{\lambda^*}$ is upper bounded by
\begin{equation}
\E_{\lambda^*, \X} \ \| \hat\beta_{\lambda^*} - \beta^* \|^2 \le \E_{\X} \ f\left( \lambda_{\min}(\X^T\X) \right),
\end{equation}
where $f$ is the following decreasing function of $\lambda_{\min}$:
\begin{equation}
f(\lambda_{\min}) = \frac{\sigma^2 \ d}{\lambda_{\min} + R} + \| \beta^* \|^2_2 \left( 1 - \frac{2 \lambda_{\min}}{R} \log\left( 1 + \frac{R}{\lambda_{\min}} \right) + \frac{\lambda_{\min}}{\lambda_{\min} + R} \right).
\end{equation}
\end{theorem}

\begin{proof}
The SVD decomposition of $\X = U S V^T$ implies that
$\X^T \X = V S U^T U S V^T = V S^2 V^T$, where $U$ and $V$ are orthogonal matrices.

We define $W = \left( \X^T \X + \lambda I \right)^{-1}$, and see that
\begin{align*}
W = \left( V (S^2 +  \lambda I ) V^T \right)^{-1} = V \ \mathrm{Diag} \left( \frac{1}{s_{jj}^2 + \lambda} \right)_{j=1}^d V^T.
\end{align*}
In this case, the MSE of $\hat\beta_\lambda$ has two sources: squared bias and the trace of the covariance matrix.
The covariance matrix of $\hat\beta_\lambda$ is
$\text{Cov}(\hat\beta_\lambda) = \sigma^2 \ W \X^T \X W,$
while its bias is given by $- \lambda W \beta^*$ (see \cite{hoerl1970ridge}).
Thus, 
\begin{equation}
\text{Cov}(\hat\beta_\lambda) = \sigma^2 \ V \ \mathrm{Diag} \ \left( \frac{{s_{jj}^2}}{{(s_{jj}^2 + \lambda)^2}} \right)_{j=1}^d V^T.
\end{equation}
Note that $s_{jj}^2 = \lambda_j$, where $s_{jj}$'s are the singular values of $\X$, and $\lambda_j$'s the eigenvalues of $\X^T\X$.
As $V$ is orthogonal, $\text{Tr} \left[ \text{Cov}(\hat\beta_\lambda) \right] = \sigma^2 \ \sum_{j=1}^d {\lambda_j}/{(\lambda_j + \lambda)^2}.$

\

Unfortunately, in practice, the value of $\lambda$ is unknown before collecting the data.
A common technique consists in using an additional \emph{validation} set to choose the optimal regularization parameter $\lambda^*$.
Generally, in supervised learning, the validation set comes from the same distribution as the test set, while in active learning it does not.
As in the unregularized case, we want to \emph{train} on unlikely data, but we want to \emph{test} on likely data.
We achieve robustness against this fact as follows.
We fix some fairly large $R > 0$ such that we assume $\lambda^* \in (0, R)$.
We treat $\lambda^*$ as a random variable, and we impose a \emph{uniform} prior $D_\lambda$ over $(0, R)$.

Then, we see that
\begin{align}
\E_{\lambda^* \sim D_\lambda} \left[ \text{Tr} \left[ \text{Cov}(\hat\beta_{\lambda^*}) \right]  \right] 
&= \sigma^2 \ \sum_{j=1}^d \lambda_j \int_0^R  \frac{1}{(\lambda_j + \lambda)^2} \frac{1}{R} \ d\lambda \nonumber \\
&= \sigma^2 \ \sum_{j=1}^d \frac{1}{\lambda_j + R} \le \frac{\sigma^2 \ d}{\lambda_{\min} + R}.
\end{align}
The squared bias can be upper bounded by
\begin{align}\label{app_ridge:squared_bias}
\lambda^2 \ {\beta^*}^T W^T W \beta^* &= {\beta^*}^T V \ \mathrm{Diag} \left[ \frac{\lambda^2}{(\lambda_j + \lambda)^2} \right]_j V^T \beta^* \nonumber \\
&\le \| \beta^* \|^2_2 \ \max_i \left( \frac{\lambda}{\lambda_j + \lambda} \right)^2 \nonumber \\
&= \| \beta^* \|^2_2 \ \left( \frac{\lambda}{\lambda_{\min} + \lambda} \right)^2.
\end{align}
for every $\lambda > 0$, as $\lambda_j \ge 0$ for all $j$.
Taking expectations on both sides of \eqref{app_ridge:squared_bias} with respect to $\lambda^* \sim D_\lambda$, and after some algebra
\begin{align}
\frac{\E_{D_\lambda} \mathrm{Bias}^2 (\hat\beta_{\lambda^*})}{\| \beta^* \|^2_2} \le 1 - \frac{2 \lambda_{\min}}{R} \log\left( 1 + \frac{R}{\lambda_{\min}} \right) + \frac{\lambda_{\min}}{\lambda_{\min} + R},
\end{align}
where the RHS is a decreasing function of $\lambda_{\min}$ that tends to zero as $\lambda_{\min}$ grows.
\end{proof}
It follows that $\E \ \| \hat\beta_{\lambda^*} - \beta^* \|^2$ can be controlled by minimizing $\lambda_{\min}(\X^T\X)$, and we can focus on minimizing $\lambda_{\min}(\bar\X^T\bar\X)$ by the equivalence shown in the \emph{Problem Definition} section of the main paper.

\section{Simulations}
\label{s:sims}
We conducted several experiments in various settings.
We present here some experiments that complement those showed in the main paper.
In particular, we show experiments for linear models, synthetic linear models, synthetic non-linear data, and additional regularized and real-world datasets.

\subsection{Linear Models}
We first empirically show the results proved in Theorem \ref{app_th:main}.
For a sequence of values of $n$, we choose $k = \sqrt{n}$ observations in $\R^d$, with fixed $d = 10$.
The observations are generated according to $\mathcal{N}(0, \mathrm{Id})$, and $y$ follows a linear model with $\beta_i \sim U(-5, 5)$.
For each tuple $(n, k)$ we repeat the experiment 200 times, and compute the squared error ($\beta^*$ is known).
The results in Figure \ref{linear_plots} (a) show the {average} MSE of Algorithm \ref{alg:threshold} significantly outperforms that of random sampling.
We also see a strong \emph{variance reduction}.
Figure \ref{linear_plots} (b) restricts the comparison to fixed and adaptive threshold algorithms; while the latter outperforms the former, the difference is small.
In Figure \ref{linear_plots} (c) we keep $n$ and $d$ fixed, and vary $k$.
Finally, in Figure \ref{nw_linear_plots} (a) we show the case where $\Sigma \neq \mathrm{Id}$.

\begin{figure*}[htp]
  \centering
  \subfigure[With $k = \sqrt{n}$, $d = 10$.]{\includegraphics[width=0.45 \columnwidth]{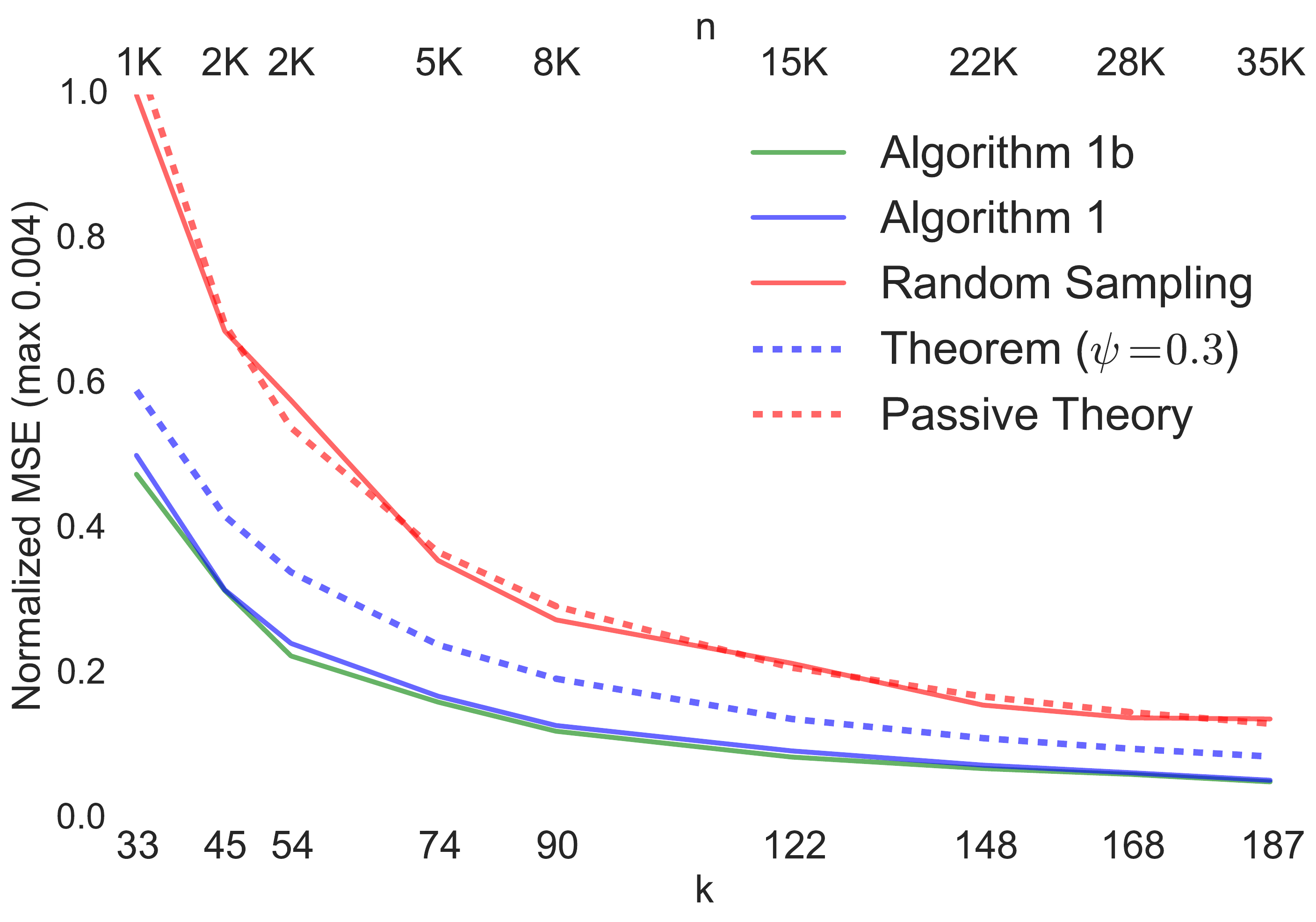}}
  \quad
      \subfigure[With $k = \sqrt{n}$, $d = 10$.]{\includegraphics[width=0.45 \columnwidth]{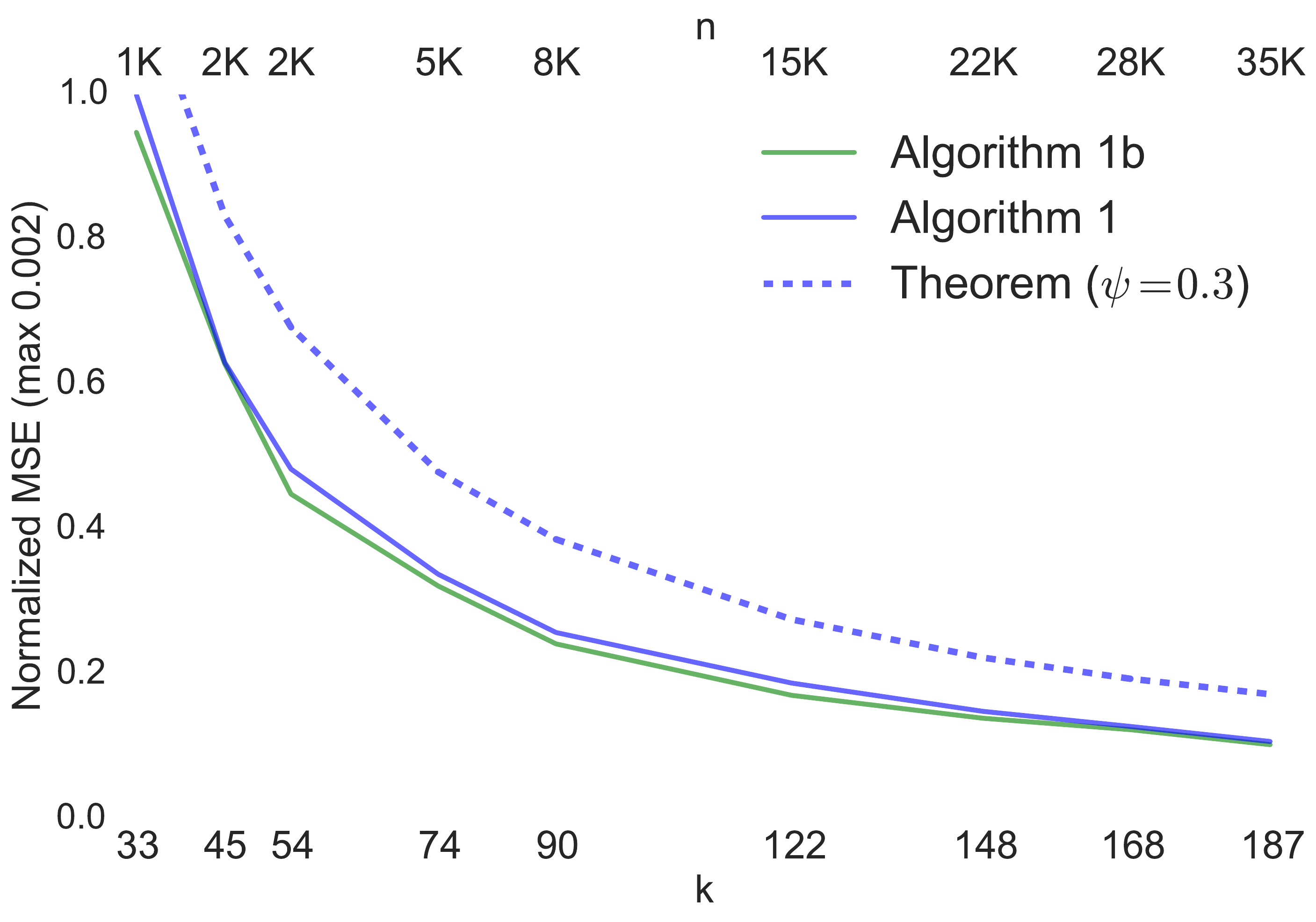}}
  \subfigure[With $n = 3000$, $d = 10$.]{\includegraphics[width=0.45 \columnwidth]{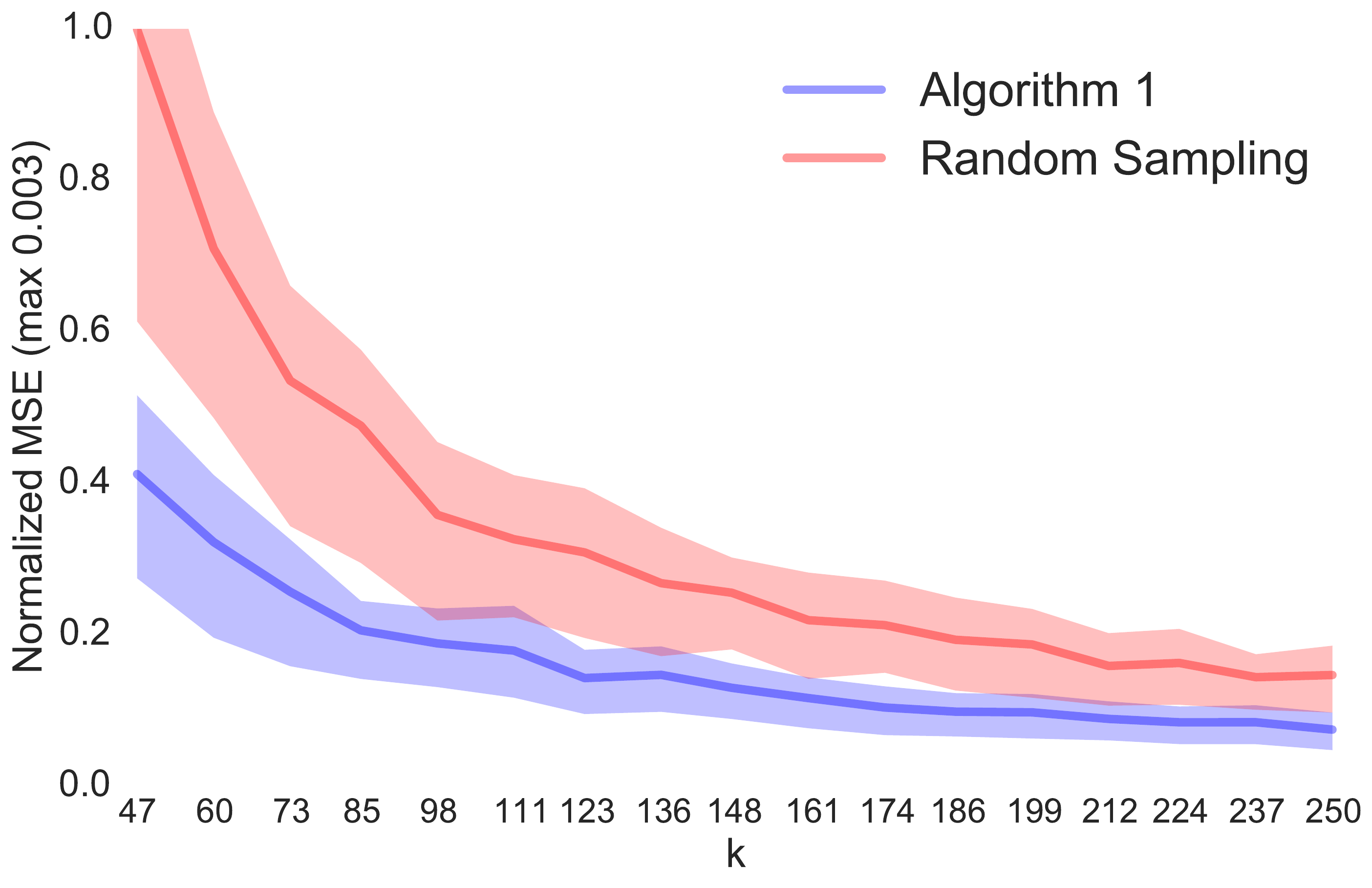}}
 \caption{$\mathrm{MSE}$ of $\hat\beta_{OLS}$; white Gaussian obs, $(0.25, 0.75)$ quantile confidence intervals displayed in (a), (c).}
 \label{linear_plots}
\end{figure*}

For completeness, we repeated the simulation with observations generated according to a joint Gaussian distribution with a random covariance matrix that had $\mathrm{Tr}(\Sigma) = 21.59$, $\lambda_{\min} = 0.65$, and $\lambda_{\max} = 3.97$.
Figure \ref{nw_linear_plots} (a) shows that thresholding algorithms outperform random sampling in a similar way as in the white case presented in the paper.
Also, Figure \ref{nw_linear_plots} (b) shows how the adaptive threshold slightly beats the fixed one.

\

Finally, in Figure \ref{nw_linear_plots} (c), we show the results of simulations when observations are sampled from Laplace correlated marginals (through a Gaussian Copula).
We compare random sampling to two versions of the thresholding algorithm.
The most simple one, denoted by Unif-Weig Algorithm, assigns uniform weights (i.e., $\xi_i = 1$ for all $i$).
On the other hand, denote by Opt-Weig Algorithm the algorithm that uses the optimal weights (previously pre-computed, in this case $\max_i \xi_i / \min_i \xi_i \approx 7$, independent variables tend to require higher weights).
As one would expect, the latter does better than the former.
However, it is remarkable that the difference between random and thresholding is way more substantial than the difference between optimal and approximate thresholding, an observation that can be very useful in practice.
\begin{figure*}[htp]
  \centering
  \subfigure[With $k = \sqrt{n}$, $d = 10$.]{\includegraphics[width=0.45 \columnwidth]{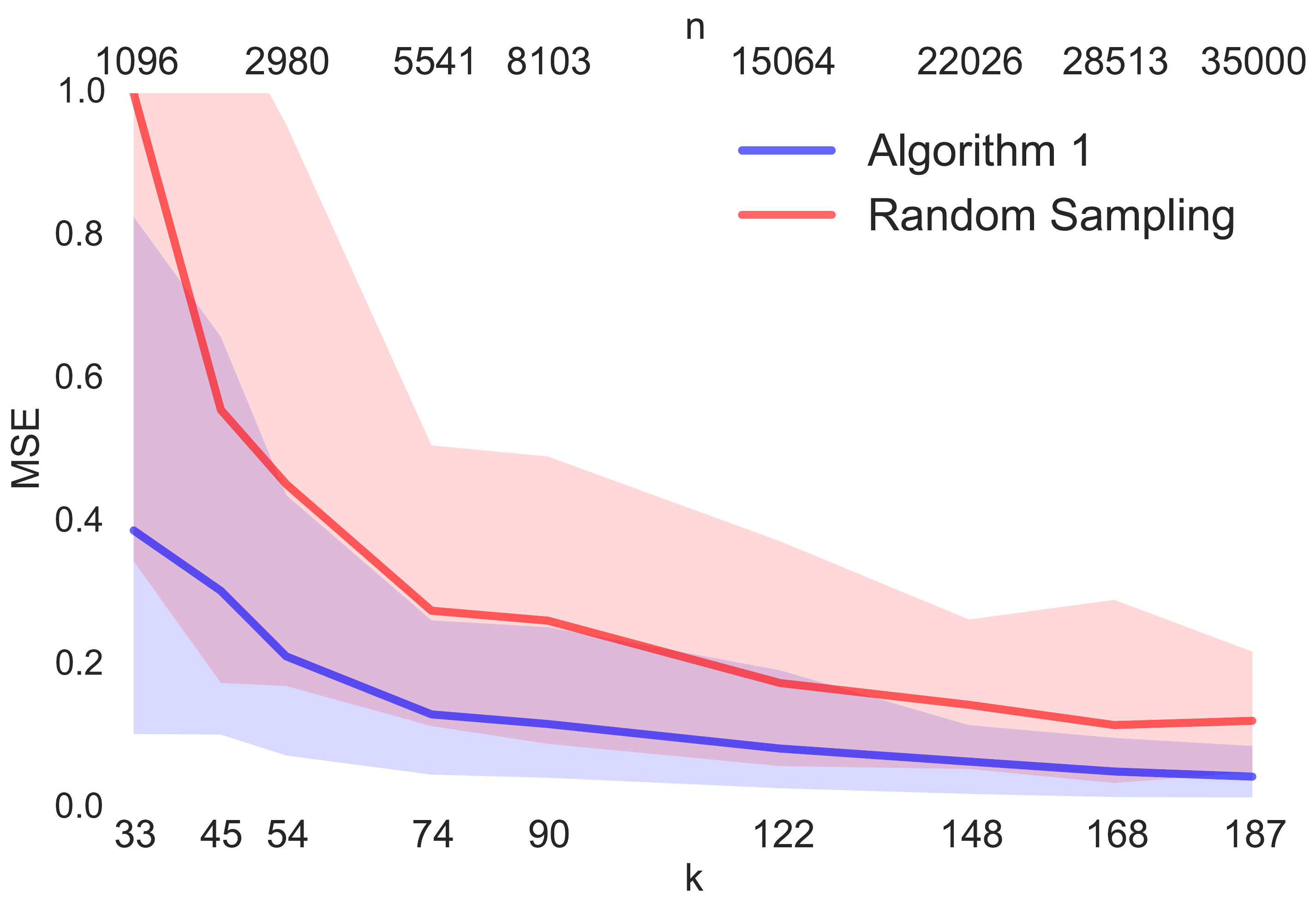}}
  \quad
      \subfigure[With $k = \sqrt{n}$, $d = 10$.]{\includegraphics[width=0.45 \columnwidth]{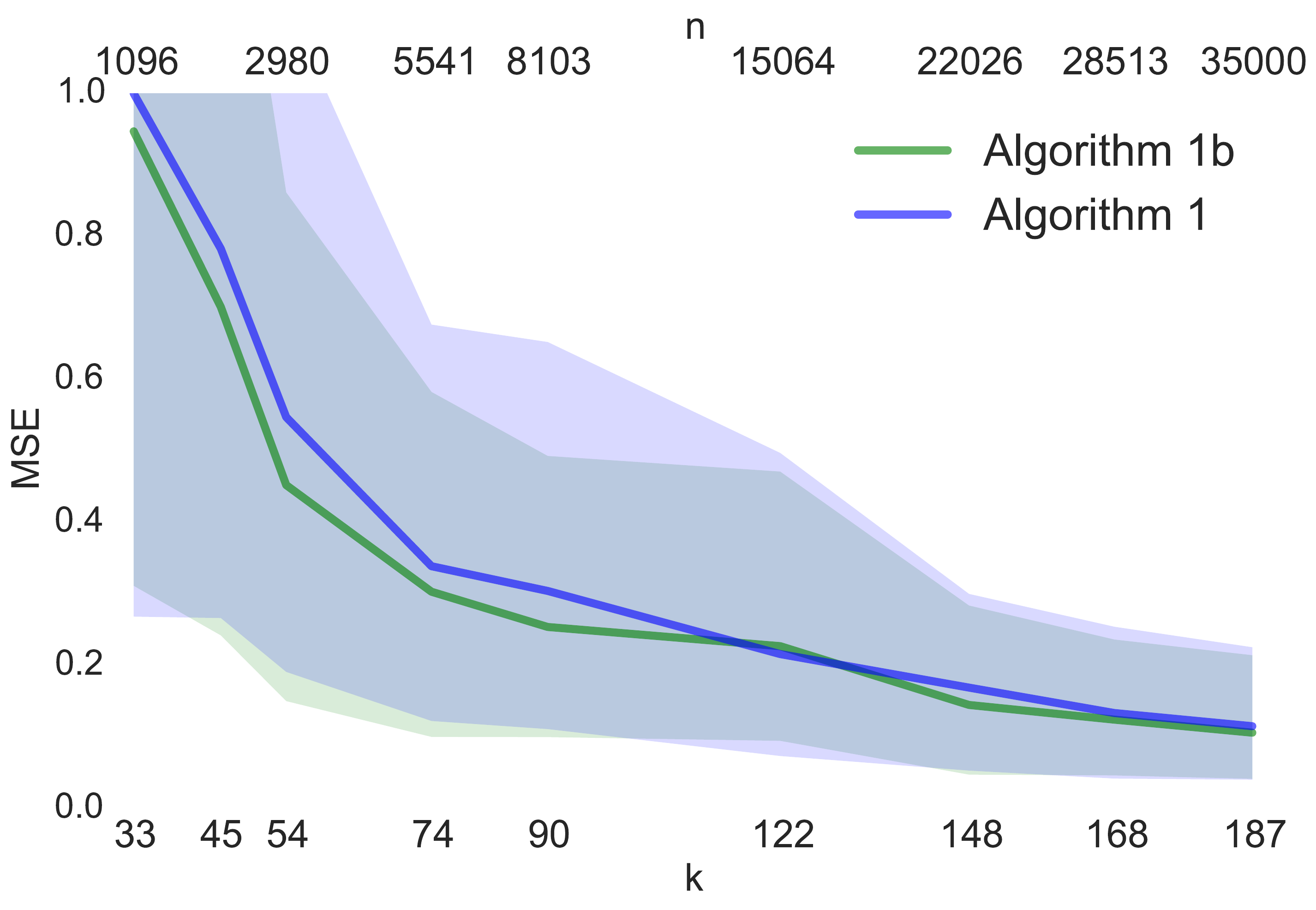}}
      \subfigure[Laplace Correlated Data via Gaussian Copula; Uniform vs.\ Optimal Weights.]{\includegraphics[width=0.45 \columnwidth]{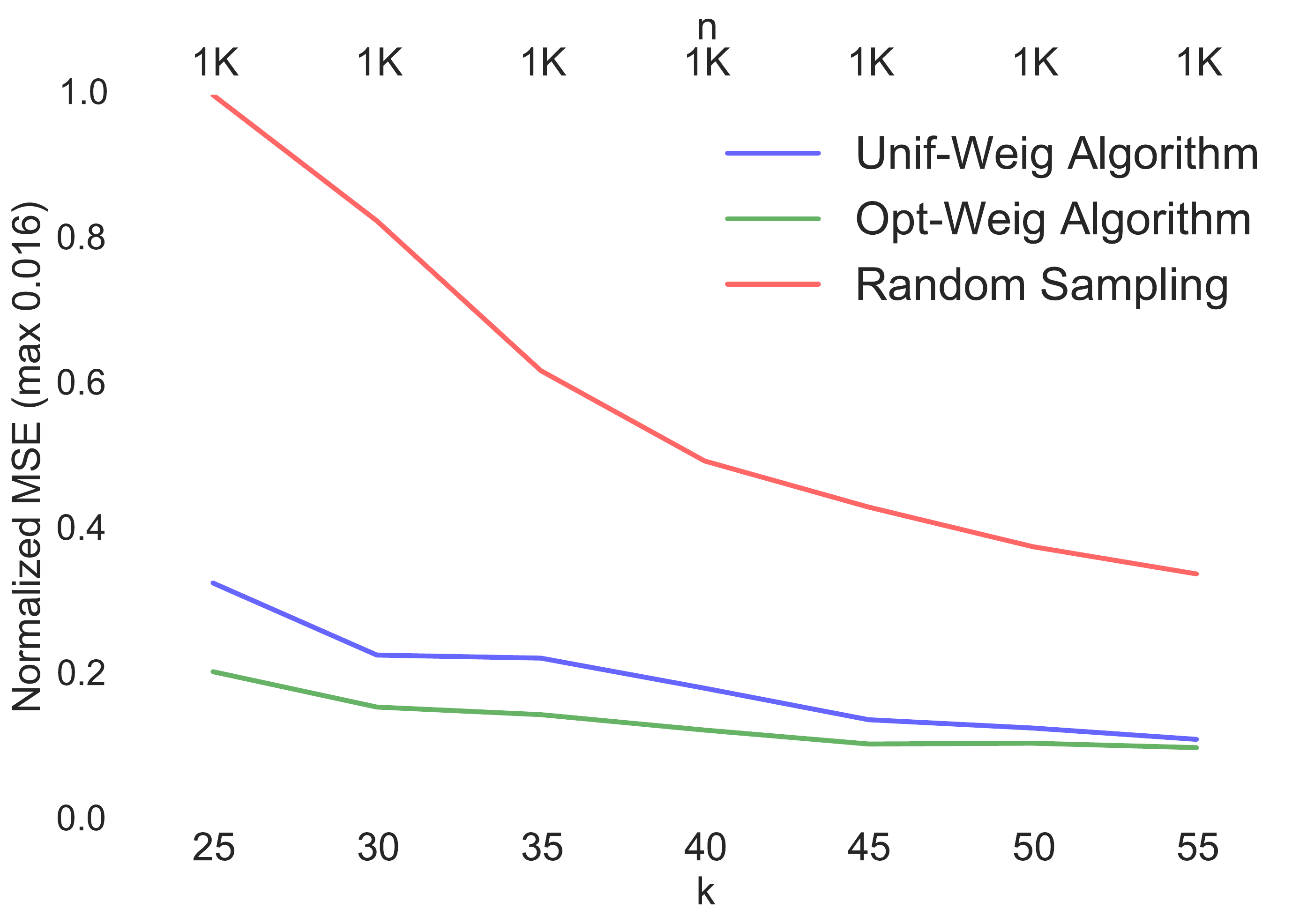}}
 \caption{In (a), (b), $\mathrm{MSE}$ of $\hat\beta_{OLS}$; $\mathcal{N}(0, \Sigma)$ data, $(0.05, 0.95)$ conf.\ intervals.}
 \label{nw_linear_plots}
\end{figure*}

\subsection{Synthetic Non-Linear Data}
The theory and algorithms presented in this paper are based on the linearity of the model. To understand the impact of this assumption, 
we perform an experiment where the response model was $y = x^T\beta + \psi x^Tx$ for various values of $\psi$, and $\beta_i \sim U(-5, 5)$.
Note that high-order terms and transformations can easily be included in the design matrix (not done in this case).
As expected, the results in Figure \ref{nonlinear-sim1} show an intersection point.
The active learning algorithms are robust to some level of non-linearity but, at some point, random sampling becomes more effective.

\begin{figure}[ht]
\vskip 0.2in
\begin{center}
\centerline{\includegraphics[width=0.6 \columnwidth]{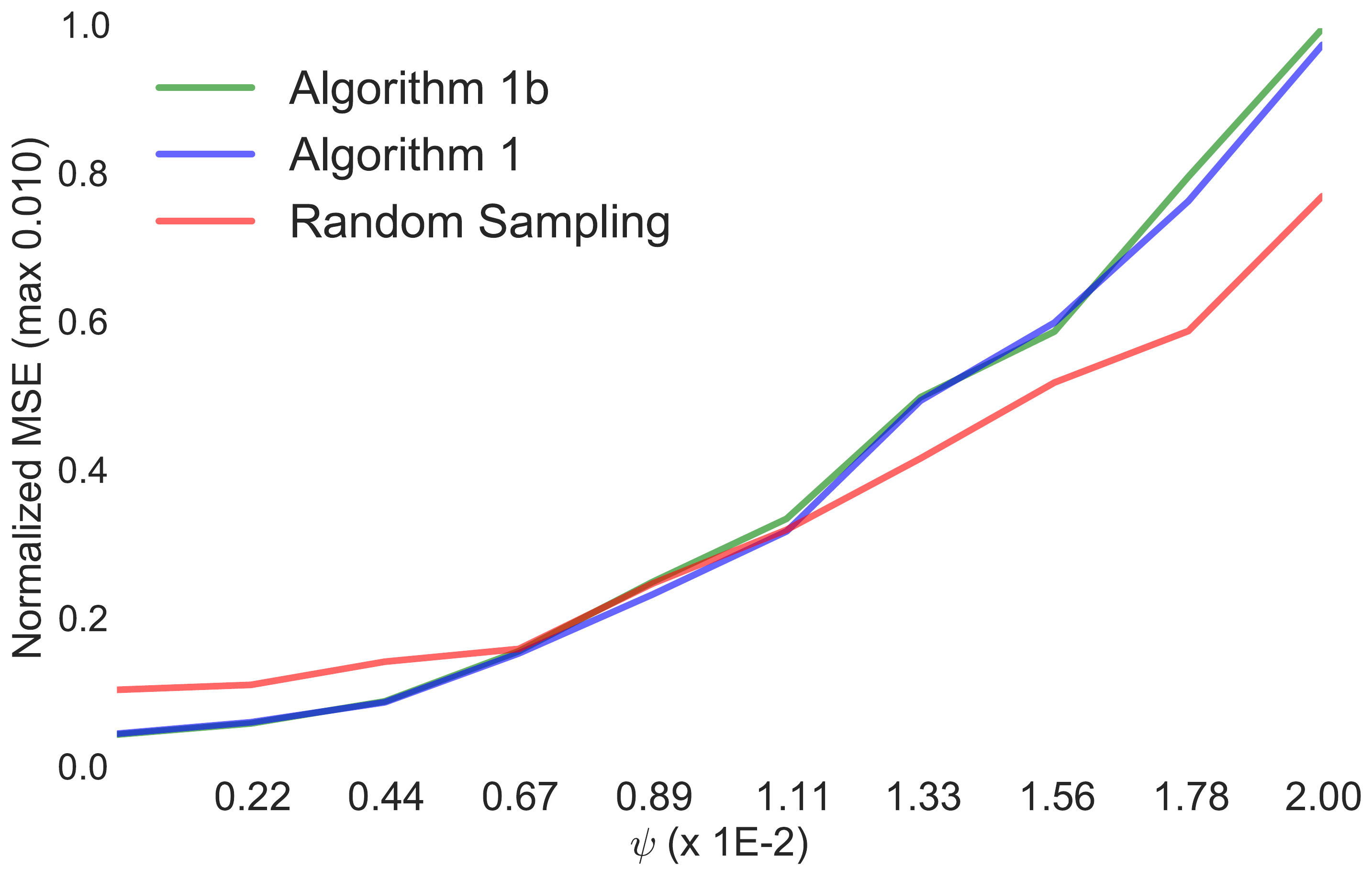}}
\caption{Model is $y = \sum_i \beta_i x_i + \psi \sum_{i} x_i^2$.}
\label{nonlinear-sim1}
\end{center}
\vskip -0.2in
\end{figure}

\subsection{Regularization}
An appealing property of the proposed algorithms is that their gain is preserved under regularized estimators such as ridge and lasso.
This is specially relevant as it allows for higher dimensional models where transformations and interactions of the original variables are added to better capture non-linearities in the data and regularization is used to avoid overfitting.
In fact, our algorithm can be thought of as a type of regularizing process.

\

\begin{figure*}[htp]
  \centering
  \subfigure[Ridge; $d = 10, k = \sqrt{n}$.]{\includegraphics[width=0.45 \columnwidth]{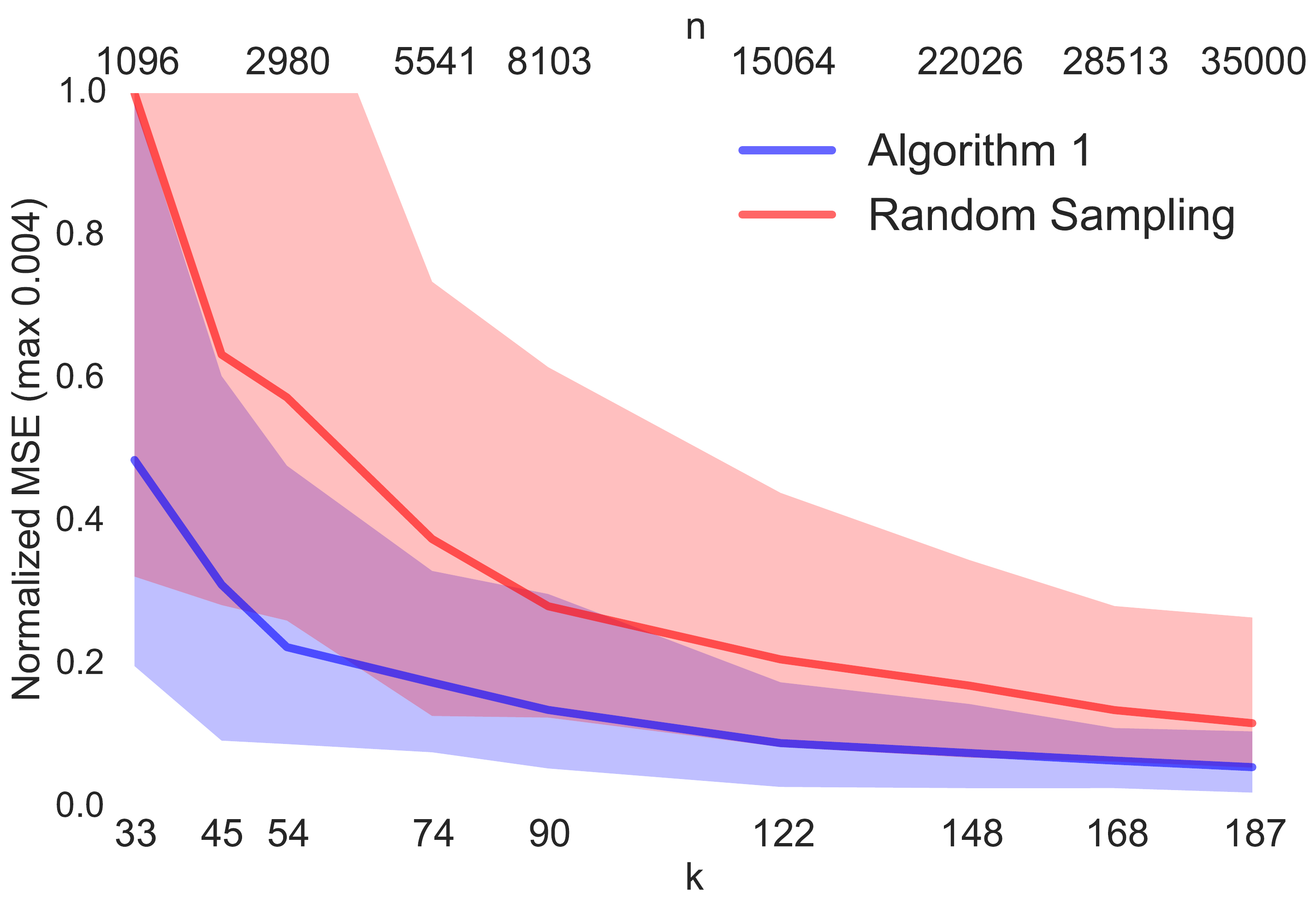}}\quad
  \subfigure[Lasso; $n = 5000, k = 150, d = 70$.]{\includegraphics[width=0.45 \columnwidth]{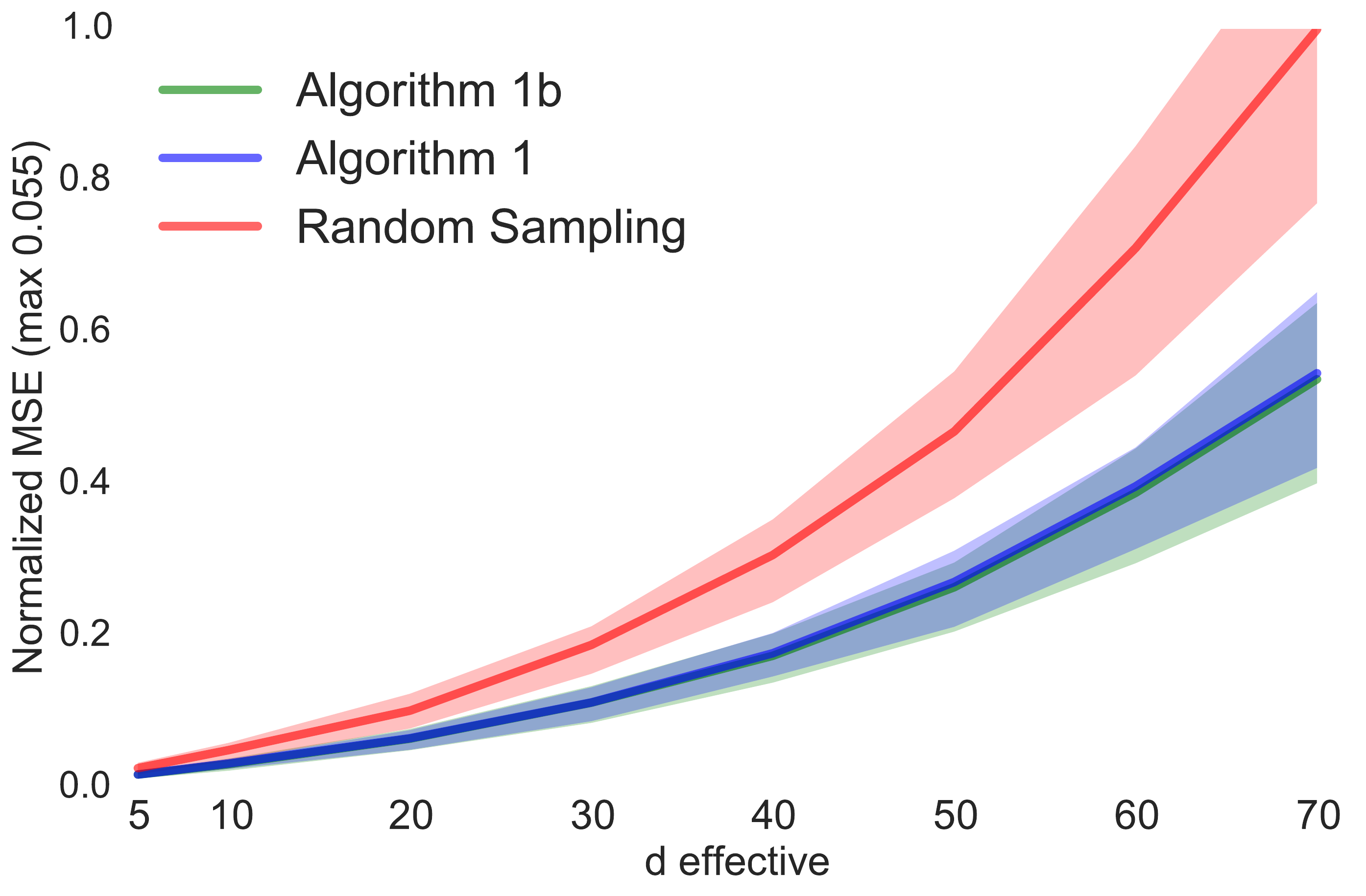}}
    \subfigure[Lasso; $n = 1000, k = 100, d_{\mathrm{eff}} = 7$]{\includegraphics[width=0.45 \columnwidth]{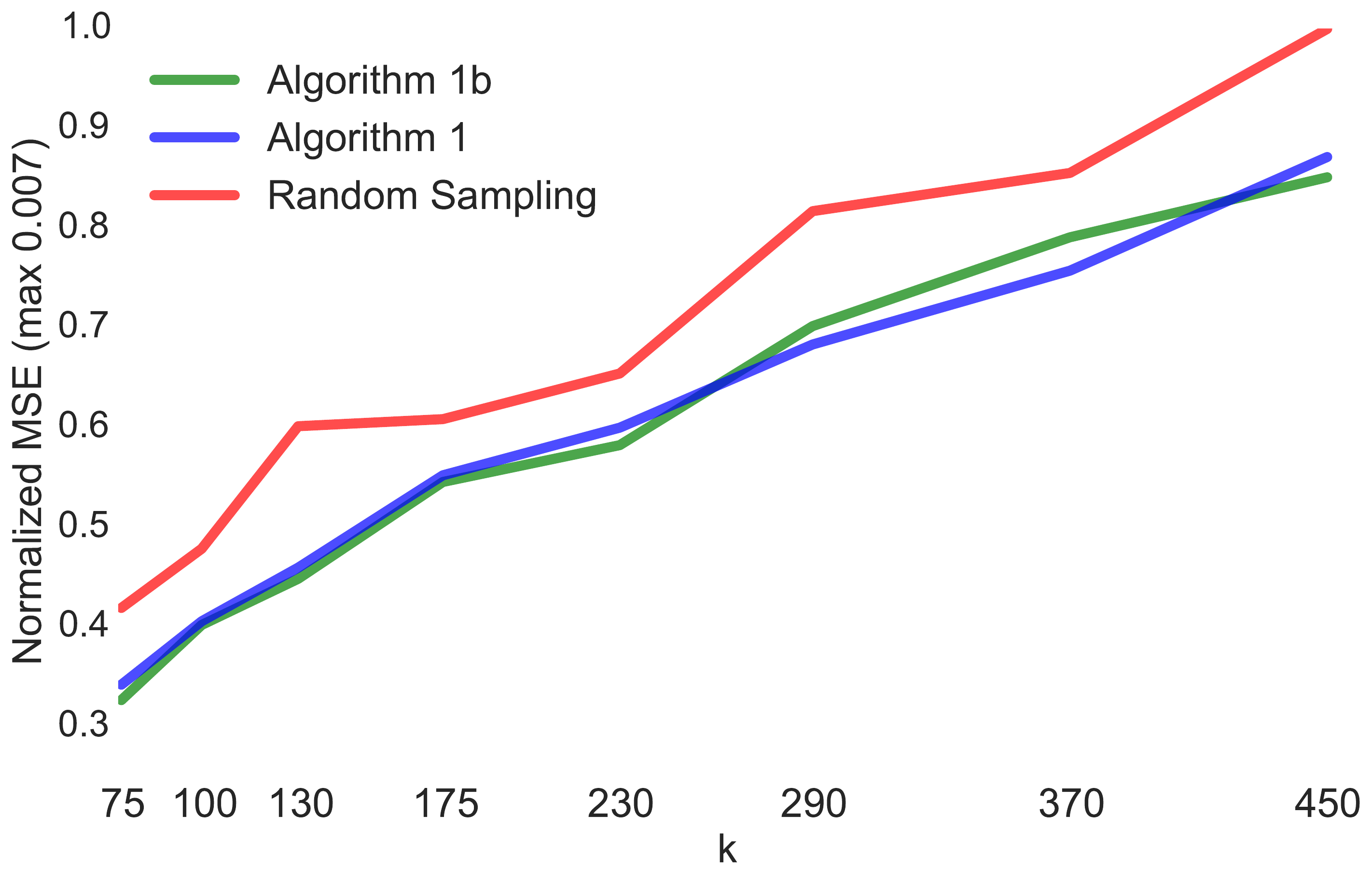}}
     \caption{$\mathrm{MSE}$ of regularized estimators, $\lambda = 0.01$; white Gaussian obs. The $(0.05, 0.95)$ confidence intervals in (a), and $(0.25, 0.75)$ in (b).}
 \label{regularized_plots}
\end{figure*}

We repeated the first experiment from the linear model simulations, using the ridge estimator with $\lambda = 0.01$.
Figure \ref{regularized_plots} (a) shows that the average MSEs of Algorithms \ref{alg:threshold} and 1b strongly outperform the results of random sampling.
Their variance is less than 30\% that of random sampling in all cases. 

\

We performed two experiments with Lasso estimators to investigate the behavior of our algorithms in the presence of \emph{sparse} models.
We do \emph{not} test  Algorithm 2 here, but only simple thresholding approaches.
First, we fixed $n = 5000$, $k = 150$, $d = 70$ and white Gaussian data.
The dimension of the latent subspace, or effective dimension of the model, ranges from $d_{\text{eff}} = 5$ to $d_{\text{eff}} = 70$.
Results are shown in Figure \ref{regularized_plots} (b).
Algorithm \ref{alg:threshold} and Algorithm 1b strongly improve the performance of random sampling, while their variance is at most half that of random sampling.

\

In the second experiment, we fixed $d_{\text{eff}} = 7$, and progressively increased the dimension of the space $d$ from $d = 70$ to $d = 450$.
Also, we kept fixed $n = 1000$ and $k = 100$.
Results are shown in Figure \ref{regularized_plots} (c).

\

Thresholding algorithms consistently decrease the MSE of the lasso estimator with respect to random sampling, even though we are adding a large number of purely noisy dimensions.
The reason is simple.
While these algorithms do not actively try to find the latent subspace (Algorithm 2 does), their observations will be, on average, larger in those dimensions too. 
There may be ways to leverage this fact, like batched approaches where weights $\xi$ are updated by giving more importance to promising dimensions.

\subsection{Real World Datasets}

The Combined Cycle Power dataset has 9568 observations.
The outcome is the net hourly electrical energy output of the plant, and it has $d = 4$ covariates: temperature, pressure, humidity, and exhaust vacuum.
In Figure \ref{rworld-sim1}, we see the phenomenon explained in the main paper (for large $k$, the gain vanishes).
In this case, and after adding all second order interactions, active learning solves the problem.
Random sampling with interactions is not shown as the error was much larger.

\begin{figure}[htp]
\vskip 0.2in
\begin{center}
\centerline{\includegraphics[width=0.6 \columnwidth]{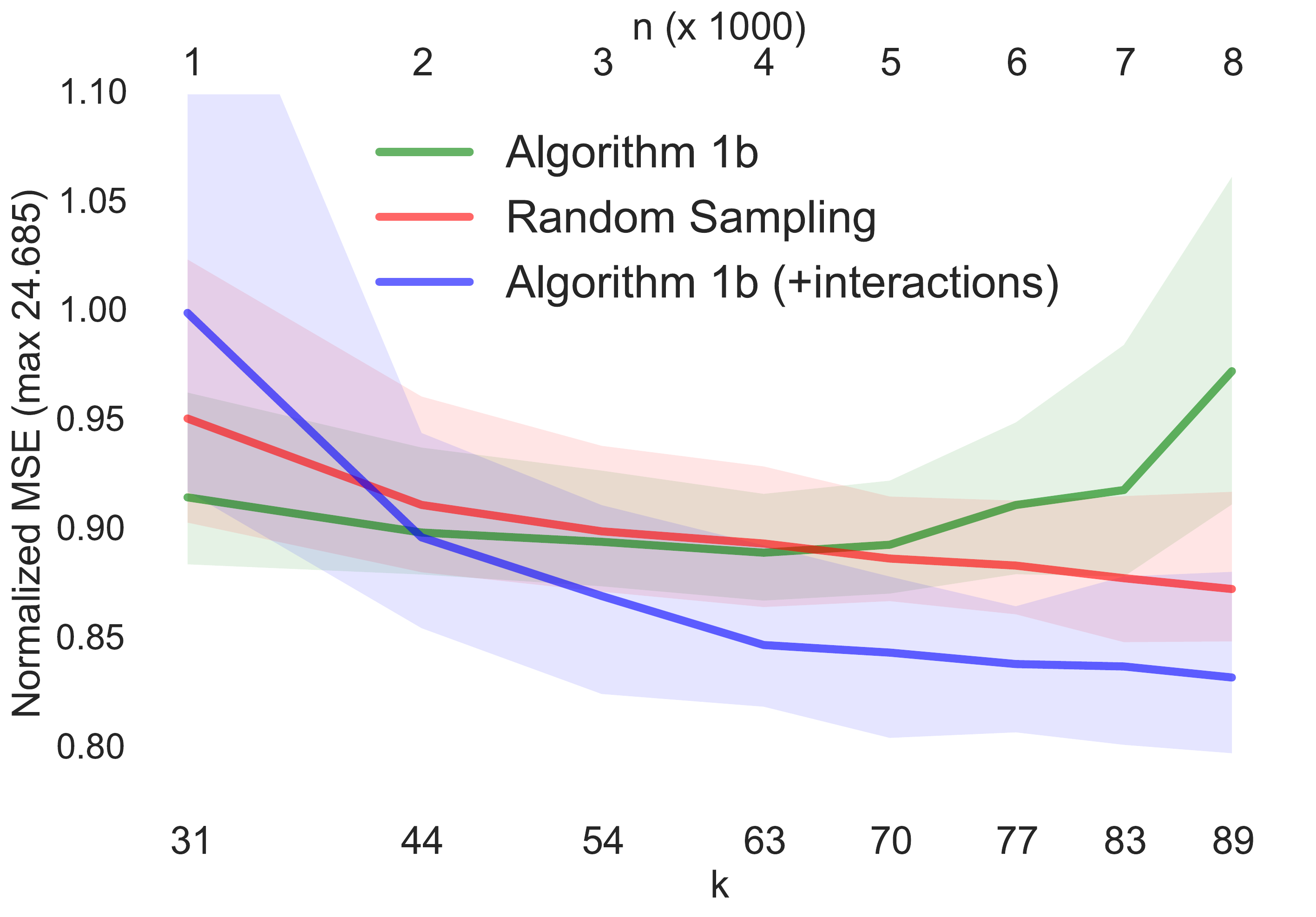}}
\caption{Combined Cycle Power (150 iters).}
\label{rworld-sim1}
\end{center}
\vskip -0.2in
\end{figure}

\

In addition, in Figure \ref{scatterplots} we show the scatterplots of the datasets used in the paper (we omitted the YearPredictionMSD dataset as $d = 90$).

\begin{figure*}[htp]
  \centering
  \subfigure[Protein Structure Dataset.]{\includegraphics[width=0.45 \columnwidth]{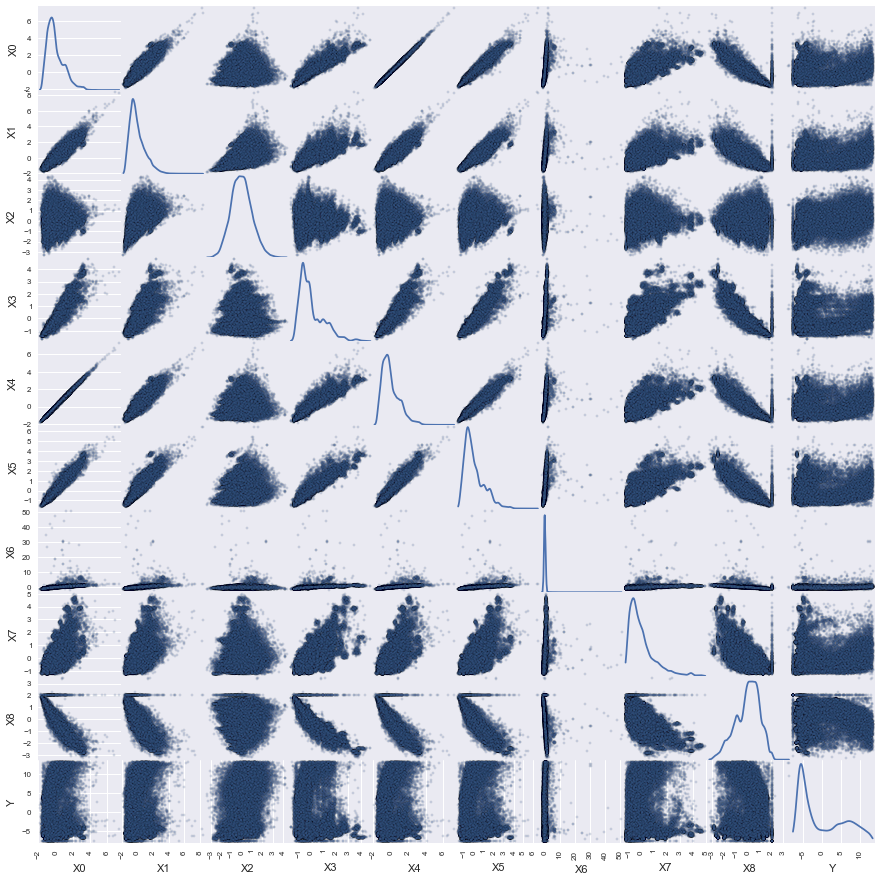}}
  \quad
  \subfigure[Combined Cycle Power Plant Dataset.]{\includegraphics[width=0.45 \columnwidth]{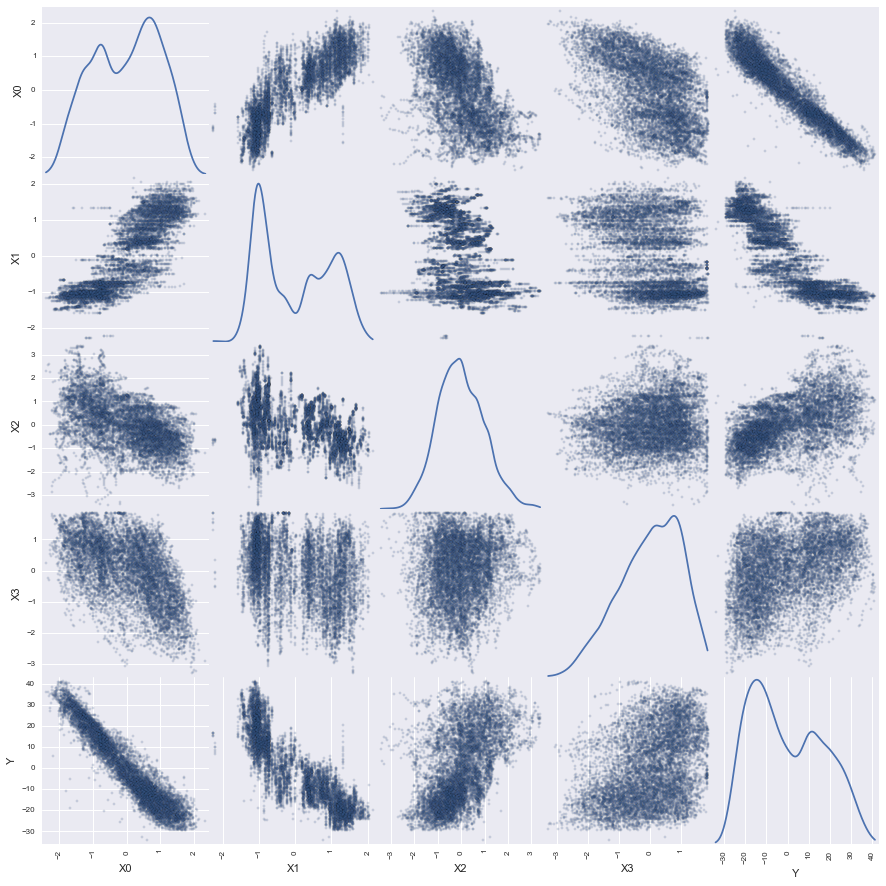}}
  \subfigure[Bike Sharing Dataset.]{\includegraphics[width=0.45 \columnwidth]{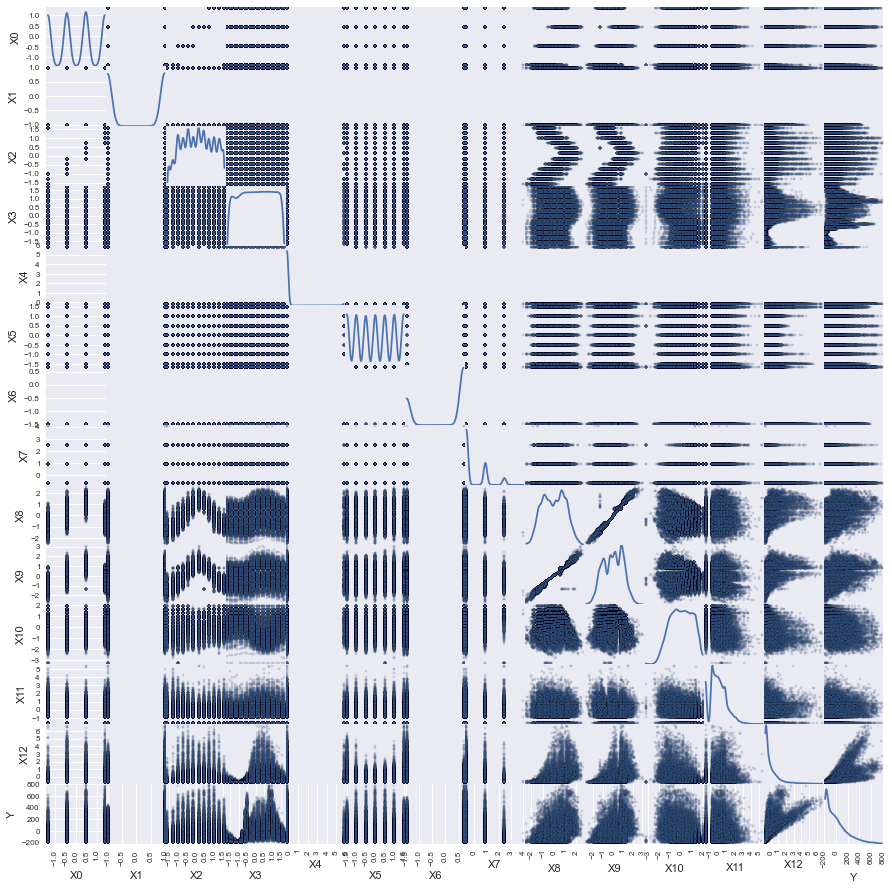}}
 \caption{Scatter Plots of Real World Datasets.}
 \label{scatterplots}
\end{figure*}

\end{document}